\documentclass{article}

\usepackage{arxiv}

\usepackage[section]{placeins}
\usepackage{float}
\usepackage[utf8]{inputenc}
\usepackage[T1]{fontenc}
\usepackage{charter}
\usepackage[light]{roboto}
\usepackage{microtype}
\RequirePackage[x11names]{xcolor}
\usepackage{colortbl}

\usepackage{amsmath}
\usepackage{amssymb}
\usepackage{graphicx}
\usepackage{url}
\usepackage{hyperref}
\usepackage{nameref}
\hypersetup{
    plainpages=false,           
    colorlinks=false,           
    pdfborder={0 0 0},          
    breaklinks=true,            
    bookmarksnumbered=true,     %
    bookmarksopen=true,         %
	hypertexnames=true,         %
}

\usepackage{multicol}
\usepackage{multirow}
\usepackage{pbox}
\usepackage{longtable}
\usepackage{booktabs}
\usepackage{tabu}
\usepackage{csvsimple}
\usepackage{array}
\usepackage[usestackEOL]{stackengine}
\usepackage[binary-units=true]{siunitx}
\usepackage{caption}
\usepackage{subcaption}
\usepackage{float}
\usepackage{paralist}
\usepackage{nicefrac}
\usepackage{algorithm}
\usepackage[noend]{algpseudocode}

\definecolor{t_blue}{HTML}{355fb3}
\definecolor{t_red}{HTML}{b33535}
\definecolor{t_green}{HTML}{3bb335}
\definecolor{t_yellow}{HTML}{b39735}
\definecolor{t_darkblue}{HTML}{1e3666}
\definecolor{t_darkgreen}{HTML}{22661e}
\definecolor{t_darkred}{HTML}{661e1e}
\definecolor{t_darkyellow}{HTML}{66571e}
\definecolor{t_lightblue}{HTML}{8ea7d7}
\definecolor{t_lightred}{HTML}{dc8989}
\definecolor{t_lightgreen}{HTML}{8ddc89}

\usepackage{pgfplots}
\usepackage{pgfplotstable}
\pgfplotsset{compat=1.14}
\usepgfplotslibrary{dateplot, statistics}
\pgfplotsset{
    cycle list={t_blue\\t_red\\t_green\\t_yellow\\},
}

\usepackage{enumitem}
\usepackage{mathtools}
\usepackage{amsthm}
\usepackage{thmtools}
\usepackage{bm}
\usepackage{bbm}
\usepackage{stmaryrd}

\makeatletter
\renewenvironment{proof}[1][\proofname]{\par
	\pushQED{\qed}%
	\trivlist%
	\item[\hskip\labelsep
		\itshape%
		#1\@addpunct{.}%
	]\ignorespaces%
}{%
  \popQED\endtrivlist\@endpefalse
}
\def\th@plain{\thm@preskip\parskip\thm@postskip0pt\itshape} 
\def\th@definition{\thm@preskip\parskip\thm@postskip0pt\normalfont}
\def\th@remark{\thm@headfont{\itshape}\normalfont\thm@preskip\parskip\thm@postskip0pt} 
\makeatother

\usepackage[capitalise,nameinlink]{cleveref}
\crefname{inv}{invariant}{invariants}
\crefname{mat}{Matrix}{Matrices}
\crefname{table}{Tbl.}{Tbls.}

\usepackage{acro}
\acsetup{first-long-format=\slshape} 
\acsetup{single}

\theoremstyle{plain}
\newtheorem{thm}{Theorem}[section]
\crefname{thm}{Thm.}{Thms.}
\Crefname{thm}{Theorem}{Theorems}
\newtheorem{prop}[thm]{Proposition}
\crefname{prop}{Prop.}{Prop.}
\Crefname{prop}{Proposition}{Propositions}
\newtheorem{lem}[thm]{Lemma}
\crefname{lem}{Lem.}{Lems.}
\Crefname{lem}{Lemma}{Lemmas}

\newtheorem{cor}[thm]{Corollary}
\crefname{cor}{Cor.}{Cors.}
\Crefname{cor}{Corollary}{Corollaries}

\theoremstyle{definition}
\newtheorem{defn}[thm]{Definition}
\crefname{defn}{Def.}{Defs.}
\Crefname{defn}{Definition}{Definitions}

\def\mean{\mathrm{mean}}
\def\wmean{\mathrm{weighted\ mean}}

\preto\tabular{\setcounter{rownumbercount}{0}}
\newcounter{rownumbercount}
\newcommand\rownumber{\stepcounter{rownumbercount}\arabic{rownumbercount}}

\usepackage[sort&compress,numbers]{natbib}

\newcommand*{\dblbrace}[2][]{#1{#2}\ifthenelse{\equal{#1}{}}{\mskip-6mu}{\mskip-8mu}#1{#2}}
\newcommand*{\ldblbrace}[1][]{\dblbrace[#1]{\{}}
\newcommand*{\rdblbrace}[1][]{\dblbrace[#1]{\}}}
\newcommand*{\colorlabel}[2]{\textsf{\textbf{\small\textcolor{#1}{#2}}}}
\newcommand*{\condbold}[3][]{\ifthenelse{\equal{#2}{1}#1}{\mathbf{#3}}{#3}}
\newcommand*{\evalres}[4][]{%
	\ifthenelse{\equal{#3}{m}}{\acs*{oom}}{%
	\ifthenelse{\equal{#3}{t}}{\acs*{oot}}{$\condbold[#1]{#2}{#3} \pm #4$}}%
}
\newcommand*\circled[2][1pt]{\tikz[baseline=(char.base)]{ 
    \node[shape=circle,draw,inner sep=#1] (char) {#2};}}
\def\gcircled#1\gcircled{\circled{\small#1}}

\title{A Novel Higher-order Weisfeiler-Lehman Graph~Convolution}
\date{}

\author{
  Clemens Damke \\
  \href{mailto:cdamke@mail.upb.de}{\texttt{cdamke@mail.upb.de}} \\
  \And
  Vitalik Melnikov \\
  \href{mailto:melnikov@mail.upb.de}{\texttt{melnikov@mail.upb.de}} \\
  \And
  Eyke Hüllermeier \\
  \href{mailto:eyke@upb.de}{\texttt{eyke@upb.de}}
}

\newcommand{\dac}[3]{\DeclareAcronym{#1}{short = #2, long = #3}}

\dac{nn}{NN}{neural network}
\dac{lm}{LRM}{logistic regression model}
\dac{svm}{SVM}{support vector machine}
\dac{mlp}{MLP}{multilayer perceptron}
\dac{cnn}{CNN}{convolutional neural network}

\dac{gcr}{GC/GR}{graph classification and regression}
\dac{gc}{GC}{graph classification}
\dac{gr}{GR}{graph regression}
\dac{gk}{GK}{graph kernel}
\dac{gnn}{GNN}{graph neural network}
\dac{gcnn}{GCNN}{graph convolutional neural network}
\dac{gcn}{GCN}{graph convolutional network}
\dac{sage}{GraphSAGE}{\textsc{sa}mple and aggre\textsc{g}at\textsc{e}}
\dac{gin}{GIN}{graph isomorphism network}
\dac{sagpool}{SAGPooling}{self-attention graph pooling}
\dac{sampool}{SAMPooling}{softmax attention mean pooling}
\dac{wlst}{WL\textsubscript{ST}}{Weisfeiler-Lehman subtree kernel}
\dac{wlsp}{WL\textsubscript{SP}}{Weisfeiler-Lehman shortest-path kernel}

\dac{ml}{ML}{machine learning}
\dac{xai}{XAI}{explainable artificial intelligence}
\dac{nlp}{NLP}{natural language processing}
\dac{gi}{GI}{graph isomorphism}
\dac{wl}{WL}{Weisfeiler-Lehman}
\dac{ft}{FT}{Fourier transform}
\dac{bfs}{BFS}{breadth-first-search}
\dac{lcm}{LCM}{lowest common multiple}
\dac{dp}{DP}{discriminative power}
\dac{cp}{CP}{computational power}

\dac{gpgpu}{GPGPU}{general purpose graphics processing unit}
\dac{nci}{NCI}{National Cancer Institute}
\dac{sse}{SSE}{secondary structure element}
\dac{oom}{OOM}{out of memory}
\dac{oot}{OOT}{out of time}
\dac{cfg}{CFG}{control-flow graph}

\begin{document}
\maketitle
{\vskip -0.2in\centering Heinz Nixdorf Institute and Department of Computer Science\\Paderborn University\vskip 0.2in}

\begin{abstract}
Current \ac*{gnn} architectures use a vertex neighborhood aggregation scheme, which limits their \acl*{dp} to that of the 1-dimensional \ac*{wl} graph isomorphism test.
Here, we propose a novel graph convolution operator that is based on the 2-dimensional \acs*{wl} test.
We formally show that the resulting 2-\acs*{wl}-\acs*{gnn} architecture is more discriminative than existing \acs*{gnn} approaches.
This theoretical result is complemented by experimental studies using synthetic and real data. 
On multiple common graph classification benchmarks, we demonstrate that the proposed model is competitive with state-of-the-art \aclp*{gk} and \acsp*{gnn}.
\end{abstract}


\section{Introduction}%
\label{sec:intro}

Graph-structured data has recently received increasing attention in \ac{ml}, with applications ranging from the prediction of chemical properties, e.g., whether a molecule is toxic~\citep{Luechtefeld2018}, to the analysis of social network structures~\citep{Fan2019} and source code~\citep{Richter2020}.
This paper focuses on the prediction of (global) graph properties, i.e., \acl{gcr}.
In order to predict a certain property of interest, for example to discriminate between graphs in a classification task, a learner must be able to \emph{detect}, either explicitly or implicitly, characteristic features of a graph that are indicative of the sought property.
To this end, suitable approaches have been developed in the fields of kernel-based machine learning and (deep) neural networks.

    \Acfp{gk} implicitly embed graphs in a kernel-induced feature space, thereby making the data\,---\,via the kernel trick\,---\,amenable to kernel-based learning methods such as \acfp{svm}. 
    The features correspond to different (local) properties of a graph.
    Examples of graph kernels include the multiscale Laplacian graph kernel~\citep{Kondor2016}, the  \acf{wlst}~\citep{Shervashidze2011}, and the \acf{wlsp}. A comprehensive and up-to-date overview is provided by \citet{Kriege2020}.
    

        Unlike \acp{gk}, \acfp{gnn} produce a prediction directly by applying so-called graph convolution layers, which iteratively aggregate vertex features.
        Nevertheless, they resemble \acp{gk} in so far as those graph convolutions typically use similar graph features, e.g., the Laplacian spectrum~\citep{Bruna2013} or the \ac{bfs} subtrees that are also used by the \acs{wlst} kernel.
		Two examples of approaches that utilize a \ac{bfs} subtree characterization are the so-called \acs{gcn}~\citep{Kipf2017} and \acs{sage}~\citep{Hamilton2017} architectures.
		Recently, \citet{Xu2018} have shown that both approaches and, more generally, all \acp{gnn} that are expressible in terms of a vertex neighborhood aggregation scheme, cannot produce different predictions for graphs that are indistinguishable via the 1-dimensional \acf{wl} graph isomorphism test.

Even though the 1-\acs{wl} test is able to distinguish almost all pairs of non-isomorphic graphs~\citep{Babai1980}, it fails at distinguishing any pair of $d$-regular graphs of the same size~\citep[Cor.~1.8.5]{Immerman1990}.
This restriction alone is not necessarily an issue in the context of \acl{gcr} problems, since in real-world domains such as social networks, graphs are rarely perfectly regular.
A more relevant restriction of 1-\acs{wl}, and therefore \acp{gnn}, is the fact that it is unable to detect cycles of length $m \geq 3$.
In practice, this means that a 1-\acs{wl}-bounded \ac{gnn} is unable to make predictions based on important domain-specific graph properties, such as the clustering coefficient of a social network or the presence of aromatic rings in molecules.

\citet{Fuerer2017} shows that the cycle detection restriction of 1-\acs{wl} does however not apply to higher dimensional generalizations of 1-\acs{wl}. More specifically, he shows that the 2-dimensional \ac{wl} test is already sufficient to count the number of $m$-cycles in any graph for all $m \leq 6$.
Recently, \citet{Arvind2019} extended this proof to a maximum cycle length of $m \leq 7$.
This advantage of 2-\acs{wl} over 1-\acs{wl} motivates the idea to define a 2-\acs{wl} inspired graph convolution operator, which is able to utilize graph characteristics that are not detectable by existing \acp{gnn}.
Such an operator will be proposed and formally analyzed in \cref{sec:wl2conv}, which is the core of this paper. 
As a preparation, we start with a brief description of the \ac{wl} test and the current 1-\acs{wl}-bounded graph convolution operators in \cref{sec:prelim}, and prove limitations of a related \acs*{gnn} extension in \cref{{sec:wl2conv:limit}}.
In \cref{sec:eval}, our novel operator is evaluated on synthetic as well as real-world graph classification datasets, prior to concluding the paper with an outlook on future research in \cref{sec:conclusion}.

\section{Preliminaries}%
\label{sec:prelim}

In this section, we introduce some important definitions as well as terminology and notation that will be used throughout the paper. 

\begin{defn}\label{defn:prelim:graph}
	A \textit{graph} $G \coloneqq (\mathcal{V}_G, \mathcal{E}_G)$ consists of a finite set of vertices $v_i \in \mathcal{V}_G$ and a set of edges $e_{ij} = (v_i, v_j) \in \mathcal{E}_G \subseteq \mathcal{V}_G^2$.
	Continuous feature vectors $x_G[v_{i}] \in \mathcal{X}_{\mathcal{V}}$, $x_G[e_{ij}] \in \mathcal{X}_{\mathcal{E}}$ may be associated with all vertices $v_i \in \mathcal{V}_G$ and edges $e_{ij}\in \mathcal{E}_G$, respectively.
	In this paper, all graphs $G$ are assumed to be undirected, i.e., ${{e_{ij} \in \mathcal{E}_G}} \leftrightarrow {e_{ji} \in \mathcal{E}_G}$ and $x_G[e_{ij}] = x_G[e_{ji}]$.
	We denote the set of all undirected graphs as $\mathcal{G}$ and the graph isomorphism relation as $G \simeq H$.
\end{defn}
\begin{defn}\label{defn:prelim:distinguish}
    Let $\phi: A \to B$ be a function with arbitrary domain $A$ and codomain $B$.
	We say that $\phi$ \textit{distinguishes} $a, a' \in A$ ($a \mathrel{\not\simeq_\phi} a'$) iff.\ $\phi(a) \neq \phi(a')$.
\end{defn}
\begin{defn}\label{defn:prelim:discr-power}
	A function $\phi: \mathcal{G} \to B$ is \textit{at least as discriminative as} $\psi: \mathcal{G} \to C$ (denoted as $\phi \succeq \psi$) iff.\ $\forall\, G, H \in \mathcal{G}: {G \mathrel{\not\simeq_\psi} H} \rightarrow {G \mathrel{\not\simeq_\phi} H}$.
	We say that $\phi$ has the \textit{same} \ac{dp} as $\psi$ (denoted as $\phi \equiv \psi$) iff. ${\phi \succeq \psi} \land {\psi \succeq \phi }$.
	Lastly, we say that $\phi$ is \textit{more discriminative} than $\psi$ (denoted as $\phi \succ \psi$) iff.\ ${\phi \succeq \psi} \land {\phi \not\equiv \psi}$.
\end{defn}

\subsection{The \acl{wl} Graph Isomorphism Test}%
\label{sec:prelim:wl}

The \acl{wl} test~\citep{Cai1992} is a popular method to distinguish graphs.
There are multiple variations of this test in the literature; we will focus on the so-called Folklore \ac{wl} test.
For a given graph $G \in \mathcal{G}$, it assigns discrete labels $c \in \mathcal{C}$, called \textit{colors}, to vertex $k$-tuples $(v_1, \dots, v_k) \in \mathcal{V}_G^k$, where $k \in \mathbb{N}$ is the so-called \textit{\ac{wl}-dimension} that can be chosen freely.
A mapping $\chi_{G, k}: \mathcal{V}_G^k \to \mathcal{C}$ is called a \textit{$k$-coloring} of $G$.
\begin{defn}
	A coloring $\chi'$ \textit{refines} $\chi$ ($\chi' \succeq \chi$) iff.\ $\forall a, b \in \mathcal{V}_G^k: {a \mathrel{\not\simeq_\chi} b} \rightarrow {a \mathrel{\not\simeq_{\chi'}} b}$, i.e.\ $\chi'$ distinguishes all tuples distinguished by $\chi$.
	They are \textit{equivalent} ($\chi \equiv \chi'$) iff.\ ${\chi \preceq \chi'} \land {\chi \succeq \chi'}$.
\end{defn}
The $k$-dimensional \ac{wl} test ($k$-\acs{wl}) is computed by iteratively refining $k$-colorings $\chi_{G, k}^{(0)} \preceq \chi_{G, k}^{(1)} \preceq \dots$ of a given graph $G$ until the convergence criterion $\chi_{G, k}^{(t)} \equiv \chi_{G, k}^{(t+1)}$ is satisfied.
We denote the final, maximally refined $k$-\acs{wl} coloring by $\hat{\chi}_{G, k}$.

In the 1-dimensional case this means that an initial color is assigned to each vertex of a graph $G$, e.g., the constant coloring $\forall v \in \mathcal{V}_G: \chi_{G,1}^{(0)}(v) = \colorlabel{t_blue}{A}$ for some initial color $\colorlabel{t_blue}{A} \in \mathcal{C}$. 
In each iteration of the 1-\acs{wl} color refinement algorithm, the following neighborhood aggregation scheme is then used to compute a new color for each vertex.
\begin{defn}\label{defn:prelim:wl1-refine}
	$\chi_{G,1}^{(t+1)}(v) \coloneqq h\left(\chi_{G,1}^{(t)}(v), \ldblbrace \chi_{G,1}^{(t)}(u)\, |\, u \in \Gamma_{G}(v) \rdblbrace\right)$,
	with $\Gamma_G(v)$ the neighboring vertices of $v \in \mathcal{V}_G$, $\ldblbrace \cdot \rdblbrace$ denoting a multiset, and $h: \mathcal{C}^* \to \mathcal{C}$ some freely-choosable injective hash function that assigns a unique color to each finite combination of colors.
\end{defn}
Analogous to the 1-dimensional refinement step from \cref{defn:prelim:wl1-refine}, the $k$-dimensional color refinement step is defined as follows.
\begin{defn}\label{defn:prelim:wlk-refine}
	$\begin{aligned}[t]
		\chi_{G,k}^{(t+1)}(s) \coloneqq h\left(\chi_{G,k}^{(t)}(s), \ldblbrace (\chi_{G,k}^{(t)}(s[u/1]), \dots, \chi_{G,k}^{(t)}(s[u/k]))\, |\, u \in \mathcal{V}_G \rdblbrace\right)\\
		\text{with } s = (v_1, \dots, v_k) \in \mathcal{V}_G^k,\quad s[u/j] \coloneqq (v_1, \dots, v_{j-1}, u, v_{j+1}, \dots, v_k) \text{.}
	\end{aligned}$
\end{defn}

In 1-\acs{wl}, a vertex color is refined by combining the colors of neighboring vertices.
In $k$-\acs{wl}, the color of a $k$-tuple $s \in \mathcal{V}_G^k$ is refined by combining the colors of its neighborhood, which is defined as the set of all $k$-tuples in which at most one vertex differs from $s$.
Note that each vertex $k$-tuple has one neighbor for each $u \in \mathcal{V}_G$, each of which is a $k$-tuple of vertex $k$-tuples.
For $k = 2$, this means that each potential edge $(v, w) \in \mathcal{V}_G^2$ has all possible walks of length $2$ from $v$ to $w$ as its neighbors.
Even though $k$-\acs{wl} refines $k$-tuple colors, lower-dimensional structures still get their own colors, since a tuple does not have to consist of distinct vertices:
In $k$-\acs{wl}, the color of a single vertex $v \in \mathcal{V}_G$ is described by $\hat{\chi}_{G, k}(s)$ for $s = (v, \dots, v) \in \mathcal{V}_G^k$;
similarly, every possible edge $e_{ij} \in \mathcal{V}_G^2$ has at least one color that can encode the adjacency information, i.e., whether $e_{ij} \in \mathcal{E}_G$.

\begin{defn}
	The color distribution $\mathit{dist}_{\chi_{G, k}}: \mathcal{C} \to \mathbb{N}_0$ of a $k$-coloring $\chi_{G, k}$ counts each color $c \in \mathcal{C}$ in the coloring, i.e., $\mathit{dist}_{\chi_{G, k}}(c) \coloneqq \left|\left\{ v \in \mathcal{V}_G^k\, |\, \chi_{G, k}(v) = c \right\}\right|$. 
\end{defn}
\begin{figure}[ht]
	\centering
	\includegraphics[width=0.8\linewidth]{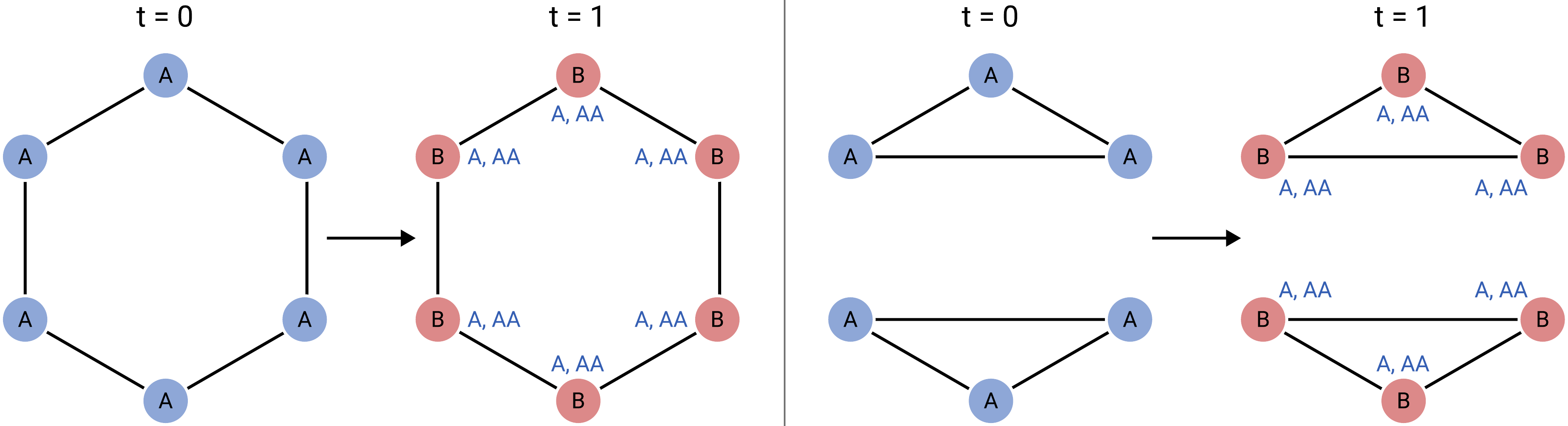}
	\caption{
		Two simple non-isomorphic graphs that are indistinguishable by 1-\acs{wl}.
	}\label{fig:prelim:wl1-problem}
\end{figure}
The output of the $k$-\acs{wl} algorithm is the color distribution $\mathit{dist}_{\hat\chi_{G, k}}$.
Since the way in which \ac{wl} colorings are refined is vertex order invariant, any difference in the final color distribution of two graphs always implies the non-isomorphism ($\not\simeq$) of the colored graphs, i.e., $G \mathrel{{\not\simeq}_{\text{$k$-\acs*{wl}}}} H \implies G \not\simeq H$.
\Cref{fig:prelim:wl1-problem} shows that the opposite does however not necessarily hold.
Additionally, it highlights the inability of 1-\acs{wl} to detect cycles of varying lengths in graphs.
\begin{defn}\label{defn:prelim:cycle-detect}
    We say that $k$-\acs{wl} \textit{detects} $m$-cycles iff.\ $k\textnormal{-\acs{wl}} \succeq d_m$, where $d_m: \mathcal{G} \to {\{ 0, 1 \}}$ is an indicator function, that determines whether a given graph contains at least one $m$-cycle.
\end{defn}
Intuitively, \cref{defn:prelim:cycle-detect} describes cycle detection as the ability to solve the corresponding decision problem given only the color distributions $\mathit{dist}_{\hat\chi_{G, k}}$ for all $G \in \mathcal{G}$.
As already mentioned in the introduction, to detect cycles of length $m \leq 7$, a \ac{wl} dimension of $k \geq 2$ is required \citep{Fuerer2017,Arvind2019}.

\subsection{Graph Neural Networks}%
\label{sec:prelim:gnn}

In this paper, we will focus on so-called spatial \acp{gnn}, which are expressible in terms of repeated vertex neighborhood aggregations.
Such a \ac{gnn} takes a graph $G \in \mathcal{G}$ with vertex feature vectors $x_G[v] \in \mathbb{R}^{d^{(0)}}$ as input; those features are typically represented as a matrix $Z_G^{(0)} \coloneqq \begin{pmatrix} x_G[v_1] \\ \vdots \\ x_G[v_n] \end{pmatrix} \in \mathbb{R}^{n \times d^{(0)}}$, where $n \coloneqq {|\mathcal{V}_G|}$.
A \ac{gnn} convolves this vertex feature matrix via a stack of graph convolution operators $S^{(t)}: \mathbb{R}^{n \times d^{(t-1)}} \to \mathbb{R}^{n \times d^{(t)}}$ s.t.\ $Z_G^{(t)} \coloneqq S^{(t)}(Z_G^{(t-1)})$.
We use $Z_G^{(t)}[v] \in \mathbb{R}^{d^{(t)}}$ to denote the row vector of $Z_G^{(t)}$, which represents the convolved vertex features of $v \in \mathcal{V}_G$.
After applying $T$ convolutional layers, the convolved vertex features $Z_G^{(T)}[v]$ can be used directly for node classification problems, or they can be combined via a pooling layer, e.g., an element-wise mean, to obtain a global graph feature vector which in turn can be used to solve \acf{gcr} problems.
In the rest of this paper, \acp{gnn} will be discussed in the context of \ac{gcr}.

To get an intuition for how \acp{gnn} relate to the \ac{wl} algorithm, one should think of the vertex feature vectors $Z_G^{(t)}$ as a continuous generalization of \ac{wl} colors $\chi_{G,1}^{(t)}$.
The graph convolution operators $S^{(t)}$ then directly correspond to \ac{wl} color refinement steps.
This intuition was recently formalized by \citet{Xu2018}, who showed the following upper bound on the \acl{dp} of \acp{gnn}.
\begin{prop}\label{prop:prelim:gcnn-wl1-limit}
	Any \ac{gnn} that convolves vertex feature vectors via a convolution operator of the form $Z_G^{(t)}[v] = h^{(t)}\left(Z_G^{(t-1)}[v], \ldblbrace Z_G^{(t-1)}[u]\, |\, u \in \Gamma_G(v) \rdblbrace\right)$ is at most as discriminative as 1-\acs{wl}, where $h^{(t)}: {\left(\mathbb{R}^{d^{(t-1)}}\right)}^{*} \to \mathbb{R}^{d^{(t)}}$ is an arbitrary vertex neighborhood hashing function.
	Moreover, iff.\ $h^{(t)}$ is injective, the \ac{gnn} has the same \ac{dp} as 1-\acs{wl}.
\end{prop}
Among others, this bound applies to the \ac{gcn}~\citep{Kipf2017} and \ac{sage}~\citep{Hamilton2017} architectures.
Since those approaches use a non-injective hashing function $h^{(t)}$, their \ac{dp} turns out to be strictly lower than that of 1-\acs{wl}.
On the other hand, the \ac{gin}~\citep{Xu2018} convolution operator achieves injectivity through the use of a \ac{mlp}, and therefore has the same \ac{dp} as 1-\acs{wl} (cf.\ \cref{defn:prelim:wl1-refine}):
\begin{align}
	Z_G^{(t)}[v] \coloneqq \mathrm{\acs*{mlp}}^{(t)}\left( (1 + \varepsilon) Z_G^{(t-1)}[v] + \smashoperator[lr]{\sum_{u \in \Gamma_G(v)}} Z_G^{(t-1)}[u] \right)%
	\quad\text{with some irrational } \varepsilon > 0 \text{.}\label{eq:prelim:gin-layer}
\end{align}

\section{Limitations of an Existing 2-\acs*{gnn}}%
\label{sec:wl2conv:limit}

The idea to extend \acp{gnn} along the lines of the higher-order \ac{wl} algorithm, which we shall elaborate on in \cref{sec:wl2conv} below, is not entirely new.
\citet{Morris2019} recently proposed the so-called $k$-\acs{gnn}, which adapts the discrete $k$-\acs{wl} refinement step (see \cref{defn:prelim:wlk-refine}) to the continuous convolution setting.
However, it turns out that $k$-\acsp{gnn} do not preserve some of the desirable properties of $k$-\acs{wl}. In particular, unlike 2-\acs{wl}, 2-\acsp{gnn} cannot count or even detect $m$-cycles in graphs.
In this section, we give a proof of this statement.

Similar to $k$-\acs{wl}, a $k$-\acs{gnn} iteratively refines/convolves the colors/features of combinations of $k$ vertices.
To reduce runtime complexity, $k$-\acsp{gnn} assign feature vectors to vertex $k$-multisets instead of $k$-tuples.
Additionally, only a ``local'' neighborhood of each multiset is considered in $k$-\acs{gnn} convolutions, whereas in $k$-\acs{wl} each tuple has a ``global'' neighborhood of size $n = |\mathcal{V}_G|$, one neighbor for each vertex $u \in \mathcal{V}_G$ (cf.\ \cref{defn:prelim:wlk-refine}).
As we will see next, the main difference between $k$-\acsp{gnn} and $k$-\acs{wl} lies in their respective notion of ``neighborhood''.
More specifically, since 2-\acs{wl} already has a significantly higher \ac{dp} than 1-\acs{wl}, we will analyze how the \ac{dp} of 2-\acsp{gnn} compares to that of 1-\acs{wl} and 2-\acs{wl}.

2-\acsp{gnn} define the neighbors of an edge $\textcolor{t_blue}{e_{ij}} = (v_i, v_j)$ to be the edges that are incident to either $v_i$ or $v_j$. 
In 2-\acs{wl}, on the other hand, the neighbors of $\textcolor{t_blue}{e_{ij}}$ are the edge pairs ${\left\{(\textcolor{t_red}{e_{il}}, \textcolor{t_darkgreen}{e_{lj}})\right\}}_{v_l \in \mathcal{V}_G}$, i.e., all possible walks of length two that start at $v_i$ and end at $v_j$. 
This difference becomes clear when comparing the definition of convolution in 2-\acsp{gnn} with that of color refinement in 2-\acs{wl}:
\begin{alignat}{4}
	\text{2-\acs{gnn}\footnotemark: } && Z_G^{(t)}[e_{ij}] &= \sigma\Biggl( \textcolor{t_blue}{Z_G^{(t-1)}[e_{ij}]} W^{(t)} + &&\left( \smashoperator[r]{\sum_{v_l \in \Gamma_G(v_j)}} \textcolor{t_red}{Z_G^{(t-1)}[e_{il}]} + \smashoperator[lr]{\sum_{{i \neq j}\,\land\, {v_l \in \Gamma_G(v_i)}}} \textcolor{t_darkgreen}{Z_G^{(t-1)}[e_{lj}]} \right) W_{\Gamma}^{(t)} \Biggr)\label{eq:wl2conv:2-gnn-layer} \\ 
	\text{2-\acs{wl}: } && \chi_{G,2}^{(t)}(e_{ij}) &= h\Bigl(\textcolor{t_blue}{\chi_{G,2}^{(t-1)}(e_{ij})}, &&\ldblbrace (\textcolor{t_red}{\chi_{G,2}^{(t-1)}(e_{il})}, \textcolor{t_darkgreen}{\chi_{G,2}^{(t-1)}(e_{lj})})\, |\, v_l \in \mathcal{V}_G \rdblbrace\Bigr) \nonumber 
\end{alignat}\footnotetext{
	Here, $\sigma$ denotes some nonlinear activation function and $W^{(t)}, W_{\Gamma}^{(t)} \in \mathbb{R}^{d^{(t-1)} \times d^{(t)}}$ the weight matrices of the convolution operator.
	To highlight the relationship between 2-\acsp{gnn} and 2-\acs{wl}, a 2-\acs{gnn} definition that uses two sums over $v_l \in \Gamma_G(v_j)$ resp.\ $v_l \in \Gamma_G(v_i)$ is shown;
	this is equivalent to a single sum over the features of the edges $\{ (u, w) \in \mathcal{E}_G \,|\, u = v_i \lor w = v_j \}$ \citep[see][]{Morris2019}.
}\begin{figure}[ht]
	\centering
	\includegraphics[width=\linewidth]{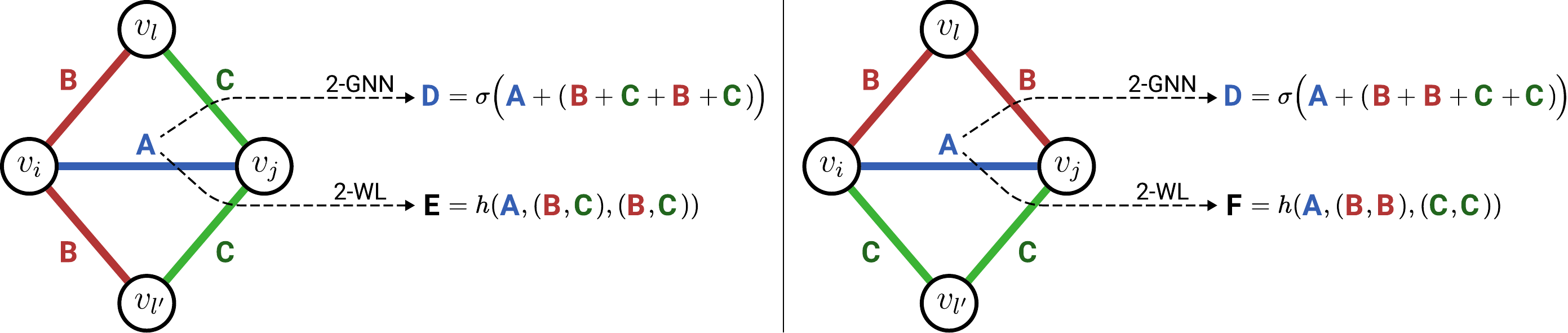}
	\caption[Edge colorings on which 2-\acsp{gnn} and 2-\acs{wl} behave differently.]{
		Two edge colorings on which 2-\acsp{gnn} and 2-\acs{wl} behave differently.
		A 2-\acs{gnn} will refine the ``color vector'' of $e_{ij}$ to \colorlabel{t_blue}{D} for both initial colorings.
		2-\acs{wl} on the other hand differentiates both colorings by preserving the color tuple information.
	}\label{fig:wl2conv:2gnn-2wl-diff}
\end{figure}

\noindent In order to analyze what those different notions of neighborhood imply for the \ac{dp} of 2-\acsp{gnn} in comparison to 2-\acs{wl}, we first show that the \ac{dp} of 2-\acsp{gnn} on all graphs $G \in \mathcal{G}$ is less than or equal to that of 1-\acs{wl} on the so-called \textit{edge neighborhood graphs} $G^{\mathcal{E}} \in \mathcal{G}^{\mathcal{E}}$.

\begin{defn}\label{defn:wl2conv:edge-graph}
	The \textit{edge neighborhood graph} of a given graph $G = (\mathcal{V}_G, \mathcal{E}_G)$ is defined as $G^{\mathcal{E}} \coloneqq (\mathcal{V}_{G^{\mathcal{E}}}, \mathcal{E}_{G^{\mathcal{E}}})$ with the vertices $\mathcal{V}_{G^{\mathcal{E}}} \coloneqq \{ \ldblbrace v, u \rdblbrace\, |\, (v, u) \in \mathcal{E}_G \lor v = u \}$ and the edges $\mathcal{E}_{G^{\mathcal{E}}} \coloneqq \left\{ (e, e') \in \mathcal{V}_{G^{\mathcal{E}}}^2 |\, \left|e \cap e'\right| = 1 \right\}$. 
\end{defn}
\begin{prop}\label{prop:wl2conv:2gnn-wl1-limit}
	The \ac{dp} of all 2-\acsp{gnn} $h_2: \mathcal{G} \to \mathcal{Y}$ is less than or equal to that of 1-\ac{wl} on edge neighborhood graphs, i.e., $\forall\, G, H \in \mathcal{G}: {G^{\mathcal{E}} \mathrel{\simeq_{\textnormal{1-\acs*{wl}}}} H^{\mathcal{E}}} \rightarrow {h_2(G) = h_2(H)}$. 
\end{prop}
\begin{proof}
	By \cref{defn:wl2conv:edge-graph}, $\Gamma_{G^{\mathcal{E}}}(e_{ij}) = {\{ (u, w) \in \mathcal{E}_G \,|\, u = v_i \lor w = v_j \}}$ for all $e_{ij} \in \mathcal{E}_G$.
	Therefore, the 2-\acs{gnn} convolution in \cref{eq:wl2conv:2-gnn-layer} can be rewritten as a vertex neighborhood convolution
	$\begin{aligned}
		Z_G^{(t)}[e] = \sigma\left( Z_G^{(t-1)}[e] W^{(t)} + \smashoperator[lr]{\sum_{e' \in \Gamma_{G^{\mathcal{E}}}(e)}} Z_G^{(t-1)}[e'] W_{\Gamma}^{(t)} \right)
	\end{aligned}$.
	\cref{prop:wl2conv:2gnn-wl1-limit} then follows from \cref{prop:prelim:gcnn-wl1-limit}.
\end{proof}
\begin{lem}\label{lem:wl2conv:wl1-regular-edge-neighbor-limit}
	1-\ac{wl} cannot distinguish the edge neighborhood graphs $G^{\mathcal{E}}$ and $H^{\mathcal{E}}$ of any pair of $d$-regular graphs $G$ and $H$ with $n$ vertices.
\end{lem}
\begin{proof}
	Let $G$ and $H$ be two $d$-regular graphs of size $n$.
	Their corresponding edge neighborhood graphs $G^{\mathcal{E}}$ and $H^{\mathcal{E}}$ both have $n^{\mathcal{E}} = n + \frac{nd}{2}$ vertices,
	$n$ of which correspond to the vertices of $G$ and $H$ respectively; we will refer to them as \textit{loop vertices} $L_G$/$L_H$.
	The remaining $\frac{nd}{2}$ edge neighborhood vertices correspond to the edges of $G$ and $H$; we will refer to them as \textit{edge vertices} $E_G$/$E_H$.

	W.l.o.g.\ we define the initial colors of the loop vertices as $\chi^{(0)}(v) = \colorlabel{t_blue}{A}$ for all $v \in {L_G \cup L_H}$. 
	The initial colors of the edge vertices are defined as $\chi^{(0)}(e) = \colorlabel{t_red}{B}$ for all $e \in E_G \cup E_H$. 
	Note that each loop vertex $\ldblbrace v_i, v_i \rdblbrace$ with $v_i \in \mathcal{V}_G \cup \mathcal{V}_H$ has $d$ neighbors, the edges incident to $v_i$.
	Similarly, each edge vertex $\ldblbrace v_i, v_j \rdblbrace$ has $2d$ neighbors, two of which are the loop vertices $\ldblbrace v_i, v_i \rdblbrace$ and $\ldblbrace v_j, v_j \rdblbrace$ with the remaining $2d - 2$ neighbors corresponding to the edges that are incident to $e_{ij}$.

	After one color refinement step, we get $\chi^{(1)}(v) = h(\colorlabel{t_blue}{A}, \ldblbrace \underbrace{\colorlabel{t_red}{B}, \dots, \colorlabel{t_red}{B}}_{d\text{ times}} \rdblbrace) \eqqcolon \colorlabel{t_blue}{C}$ for all loop vertices $v \in  L_G \cup L_H$ and $\chi^{(1)}(e) = h(\colorlabel{t_red}{B}, \ldblbrace \colorlabel{t_blue}{A}, \colorlabel{t_blue}{A}, \underbrace{\colorlabel{t_red}{B}, \dots, \colorlabel{t_red}{B}}_{\mathclap{2d - 2 \text{ times}}} \rdblbrace) \eqqcolon \colorlabel{t_red}{D}$. 
	This means that $\chi^{(0)}$ and $\chi^{(1)}$ are identical up to the color substitutions $\colorlabel{t_blue}{A} \to \colorlabel{t_blue}{C}$ and $\colorlabel{t_red}{B} \to \colorlabel{t_red}{D}$, i.e.\ $\chi^{(0)} \equiv \chi^{(1)}$, which in turn implies that 1-\ac{wl} terminates after one iteration. 
	\Cref{lem:wl2conv:wl1-regular-edge-neighbor-limit} then directly follows, since both $G^{\mathcal{E}}$ and $H^{\mathcal{E}}$ have $n$ vertices with the final color $\colorlabel{t_blue}{C}$ and $\frac{nd}{2}$ vertices with the final color $\colorlabel{t_red}{D}$, i.e.\ $G^{\mathcal{E}} \mathrel{\simeq_{\textnormal{1-\acs*{wl}}}} H^{\mathcal{E}}$. 
\end{proof}
\begin{prop}\label{prop:wl2conv:2gnn-regular-limit}
	A 2-\acs{gnn} cannot distinguish regular graphs of the same size and therefore has a lower \ac{dp} than 2-\ac{wl}.
\end{prop}
\begin{proof}
	The proposition directly follows from \cref{prop:wl2conv:2gnn-wl1-limit}, \cref{lem:wl2conv:wl1-regular-edge-neighbor-limit} and the fact that 2-\acs{wl} is able to distinguish most regular graphs~\citep[Cor.~1.8.6]{Immerman1990}.
\end{proof}

As previously mentioned, the \ac{dp} of a model by itself is not necessarily relevant for real-world \ac{gcr} problems.
However, 2-\acs{wl} is not only more discriminative than 1-\acs{wl}, but is also able to detect and count the number of $m$-cycles in a given graph for all $m \leq 7$.
We now show that 2-\acsp{gnn} not only have a lower \ac{dp} than 2-\acs{wl}, but are also unable to detect cycles.
\begin{prop}\label{prop:wl2conv:2gnn-cycle-limit}
	2-\acsp{gnn} cannot detect $m$-cycles for $m \geq 3$.
\end{prop}
\begin{proof}
	Let $n$ be the \ac{lcm} of $3$ and some $m > 3$.
	We define $c_3 \coloneqq \frac{n}{3}$ and $c_m \coloneqq \frac{n}{m}$.
	Based on that, we define the following two graphs:
	Let $G_3$ be a graph consisting of $c_3$ disconnected cycles of length $3$, analogously let $G_m$ be a graph consisting of $c_m$ disconnected cycles of length $m$.
	Since both $G_3$ and $G_m$ are 2-regular and have the size $n$, any 2-\ac{gnn} $h_2: \mathcal{G} \to \mathcal{Y}$ must map both of them to the same $y \in \mathcal{Y}$ by \cref{prop:wl2conv:2gnn-regular-limit}, i.e., $G_3 \mathrel{\simeq_{h_2}} G_m$.

	Let us assume that $h_2$ is able to detect cycles of length $3$, i.e.\ triangles.
	Following \cref{defn:prelim:cycle-detect}, this would imply that $h_2$ is at least as discriminative as the triangle detection function $d_3: \mathcal{G} \to \{ 0, 1\}$.
	It follows that ${d_3(G_3) = 1} \land {d_3(G_m) = 0} \Rightarrow G_3 \mathrel{\not\simeq_{d_3}} G_m \xRightarrow{\raisebox{0.5pt}{\scriptsize$h_2 \succeq d_3\ $}} G_3 \mathrel{\not\simeq_{h_2}} G_m$, which is a contradiction.
	Conversely, assuming that $h_2$ is able to detect cycles of length $m > 3$, the $m$-cycle detection function $d_m$ also distinguishes $G_3$ and $G_m$, which again results in the contradiction $G_3 \mathrel{\simeq_{h_2}} G_m \land G_3 \mathrel{\not\simeq_{h_2}} G_m$.
\end{proof}


\section{The 2-\acs*{wl} Graph Convolution Operator}%
\label{sec:wl2conv}

In the previous section, we compared 2-\acsp{gnn} with the 2-\acs{wl} algorithm and found that the former have a significantly lower \ac{dp} than the latter.
Motivated by this limitation, we devote this section to a novel, more discriminative convolution operator, which is inspired by the higher-order \ac{wl} algorithm and meant to overcome the limitations of 1-\acs{wl}.
Our operator is inspired by 2-\acs{wl} but uses the following simplifications to reduce its computational cost.

Similar to $k$-\acsp{gnn}, or more specifically, 2-\acsp{gnn}, our novel operator refines/convolves the feature vectors of \emph{2-multisets} $\ldblbrace v_i, v_j \rdblbrace$ instead of refining/convolving the feature vectors of 2-tuples $(v_i, v_j)$.
This simplification halves the number of feature vectors without affecting the \ac{dp} because we assume that graphs are undirected, i.e.\ $e_{ij}$ and $e_{ji}$ have identical feature vectors $x[e_{ij}] = x[e_{ji}] \in \mathcal{X}_{\mathcal{E}}$ and the same 2-\acs{wl} neighborhood.
To simplify the notation, we assume that $e_{ij} = e_{ji} = \ldblbrace v_i, v_j \rdblbrace$ in the rest of the paper.
	
After applying the 2-multiset simplification, the 2-\acs{wl} algorithm refines the color of all multisets $e_{ij} \in \mathcal{V}_G^2$ by hashing its current color and the colors of all neighbors ${\{ \ldblbrace e_{il}, e_{lj} \rdblbrace \}}_{v_l \in \mathcal{V}_G}$.
This means that the time complexity of a single refinement step is $\mathcal{O}(n^3)$ for $n \coloneqq \left| \mathcal{V}_G \right|$, which quickly becomes infeasible for large graphs.
To address this issue, we reduce both the number of colored edges as well as the number of neighbors of each edge.
This is achieved by only considering the edges that are part of the so-called \textit{$r$-th power of a graph $G$}, where $r \in \mathbb{N}$ is the freely choosable \textit{neighborhood radius}.
\begin{defn}
	The \textit{$r$-th power of a graph $G$} is defined as
	\begin{align*}
		G^r \coloneqq \left(\mathcal{V}_G, \left\{ e_{ij} \in \mathcal{V}_G^2\, |\, d_{\mathrm{SP}, G}(v_i, v_j) \leq r \right\}\right) \, , 
	\end{align*}
	where $d_{\mathrm{SP}, G}(v_i, v_j)$ is the length of the shortest path between $v_i$ and $v_j$ in $G$.
	The distance of a vertex $v_i \in \mathcal{V}_G$ to itself is defined as $d_{\mathrm{SP}, G}(v_i, v_i) \coloneqq 0$.
	Note that $G^1$ does not generally equal $G$ because $G^1$ has self-loop edges $e_{ii} \in \mathcal{E}_{G^1}$ at all vertices.
\end{defn}
For the neighborhood radius $r = 1$, only the self-loop edges ${\{ e_{ii} \}}_{v_i \in \mathcal{V}_G}$ and the edges $\mathcal{E}_{G}$ are considered;
for $r > 1$, edges between indirectly connected vertices are considered as well.
Through the reduction of the considered edges, the neighbors of each $e_{ij} \in \mathcal{E}_{G^r}$ are in turn reduced to the common $r$-neighbors of $v_i$ and $v_j$, i.e.\ $\{ \ldblbrace e_{il}, e_{lj} \rdblbrace\, |\, v_l \in {\Gamma_{G^r}(v_i) \cap \Gamma_{G^r}(v_j)} \}$. 

Let us now consider what the reduced number of used edges and the reduced number of edge neighbors implies for the runtime of a refinement step.
If $G$ is a sparse graph with the maximum vertex degree $d \coloneqq \max_{v \in \mathcal{V}_G} \left|\Gamma_G(v)\right|$, the number of considered edges is bounded by $\mathcal{O}(n d^r)$, where each edge has at most $\mathcal{O}(d^r)$ neighbors.
Consequently, the time complexity of a refinement step becomes $\mathcal{O}(n d^{2r})$, which is a significant improvement over the $\mathcal{O}(n^3)$ bound of a full 2-\acs{wl} refinement step (assuming $d \ll n$).

Based on the 2-multiset and the neighborhood localization simplifications, we now define the 2-\ac{wl} convolution operator and the corresponding \textit{2-\acs{wl}-\acs{gnn}}.
\begin{defn}\label{defn:wl2conv:wl2-conv-init}
	The \textit{initial feature matrix $Z_G^{(0)}$} of the 2-\acs{wl} convolution operator with the neighborhood radius $r \in \mathbb{N}$ contains both the vertex features $x[v_i] \in \mathcal{X}_{\mathcal{V}} = \mathbb{R}^{d_{\mathcal{V}}}$ as well as the edge features $x[e_{ij}] \in \mathcal{X}_{\mathcal{E}} = \mathbb{R}^{d_{\mathcal{E}}}$ of a given graph $G$.
	More specifically, $Z_G^{(0)} \in \mathbb{R}^{{|\mathcal{E}_{G^r}|} \times (d_{\mathcal{V}} + d_{\mathcal{E}})}$ assigns a row vector $Z_G^{(0)}[e_{ij}]$ to all edges $e_{ij} \in \mathcal{E}_{G^r}$.
	Those initial edge feature vectors are defined by the following vector concatenation ($\oplus$):
	\begin{align*}
		Z_G^{(0)}[e_{ij}] \coloneqq \left( \begin{cases}
			x[v_i] & \text{if $i = j$} \\
			\mathbf{0} & \text{else}
		\end{cases} \right) \oplus \left( \begin{cases}
			x[e_{ij}] & \text{if $e_{ij} \in \mathcal{E}_G$} \\
			\mathbf{0} & \text{else}
		\end{cases} \right)
	\end{align*}
\end{defn}
\begin{defn}\label{defn:wl2conv:wl2-conv-step}
	We define the \textit{2-\acs{wl} graph convolution operator} as
	\begin{gather*}
		Z_G^{(t)}[e_{ij}] \coloneqq \sigma\left( Z_G^{(t-1)}[e_{ij}] W_{\mathrm{L}}^{(t)} + \smashoperator[lr]{\sum_{v_l \in {\Gamma_{G^r}(v_i) \cap \Gamma_{G^r}(v_j)}}} \kappa^{(t)}\left( Z_G^{(t-1)}[e_{ij}], \ldblbrace Z_G^{(t-1)}[e_{il}], Z_G^{(t-1)}[e_{lj}] \rdblbrace \right) \right) \, , \\
		\text{with } \kappa^{(t)}(z_{ij}, \ldblbrace z_{il}, z_{lj} \rdblbrace) \coloneqq \left( z_{ij} W_{\mathrm{F}}^{(t)} \right) \odot \sigma_{\Gamma}\left(\left( z_{il} + z_{lj} \right) W_{\Gamma}^{(t)} \right)
		\text{.}
	\end{gather*}
\end{defn}
This operator is parameterized by the three matrices $W_{\mathrm{L}}^{(t)}, W_{\mathrm{F}}^{(t)},  W_{\Gamma}^{(t)} \in \mathbb{R}^{d^{(t-1)} \times d^{(t)}}$ and uses two freely choosable activation functions $\sigma$ and $\sigma_{\Gamma}$.
We use $\odot$ to denote element-wise multiplication.
In the following, we will analyze the \ac{dp} of \acp{gnn} using the 2-\acs{wl} convolution operator.
Our goal is to show that such 2-\acs{wl}-\acsp{gnn} have a strictly higher \ac{dp} than 1-\acs{wl}.
We begin by proving that 2-\acs{wl}-\acsp{gnn} are at least as discriminative as 1-\acs{wl}.
\begin{defn}\label{defn:wl2conv:vert-conv}
	A \ac{gnn} $h_1: \mathcal{G} \to \mathcal{Y}$ uses \textit{weighted vertex neighborhood sums} if its convolutional layers can be described by
	\begin{align*}
		Z_{G}^{(t)}[v_i] = \mathrm{\acs*{mlp}}^{(t)}\left( w_{ii} Z_{G}^{(t-1)}[v_i] + \smashoperator[lr]{\sum_{v_j \in \Gamma_G(v_i)}} w_{ij} Z_{G}^{(t-1)}[v_l] \right)
		\text{.}
	\end{align*}
\end{defn}

\noindent\Cref{defn:wl2conv:vert-conv} includes \acp{gcn}~\citep{Kipf2017}, where the \ac{mlp} only consists of a single layer with weights $w_{ij} = {\left(\left| \Gamma_G(v_i) \right| + 1\right)}^{-\frac{1}{2}}{\left(\left| \Gamma_G(v_j) \right| + 1\right)}^{-\frac{1}{2}}$.
\Acp{gin} also trivially satisfy the definition (see \cref{eq:prelim:gin-layer}).
\begin{thm}\label{thm:wl2conv:wl2-simulation}
	For each \ac{gnn} $h_1$ using weighted vertex neighborhood sums, there is a 2-\acs{wl}-\acs{gnn} $h_2$ that simulates $h_1$, i.e., such that $\forall G \in \mathcal{G}: h_1(G) = h_2(G)$.
\end{thm}
\begin{proof}
	We prove by construction.
	Let $G \in \mathcal{G}$ be an arbitrary input graph with $n \coloneqq \left|\mathcal{V}_G\right|$ and $m \coloneqq n + \left|\mathcal{E}_G\right|$.
	By definition, $h_1$ is a \ac{gnn} of the form $\mathit{Pool}_1(\mathit{Conv}_1(G))$, where $\mathit{Conv}_1$ is a stack of $T$ weighted vertex neighborhood sum convolutions ${\left\{ S^{(t)}: \mathbb{R}^{n \times d^{(t-1)}} \to \mathbb{R}^{n \times d^{(t)}} \right\}}_{t=1}^T$ with each corresponding $\mathrm{\acs*{mlp}}^{(t)}$ having $K$ layers. 
	$\mathit{Pool}_1$ combines the vertex feature vectors produced by $\mathit{Conv}_1$.
	Let $h_2$ be a \ac{gnn} of the form $\mathit{Pool}_2(\mathit{Conv}_2(G))$, where $\mathit{Conv}_2$ is a stack of $(2 + K) T$ 2-\acs{wl} convolution layers ${\left\{ S^{(t, k)}: \mathbb{R}^{m \times d^{(t,k-1)}} \to \mathbb{R}^{m \times d^{(t,k)}} \right\}}_{(t, k) \in [T] \times [2+K]}$ (see \cref{fig:wl2conv:wl2-simulation}) with the neighborhood radius $r \coloneqq 1$. 
	The layers ${\left\{ S^{(t, 2 + K)} \right\}}_{t=1}^T$ produce the feature matrices $Z^{(t, 2 + K)} = Z^{(t + 1, 0)}$ which are fed as input into the successor layer $S^{(t + 1, 1)}$. 
	\begin{figure}[ht]
		\centering
		\includegraphics[width=0.8\linewidth]{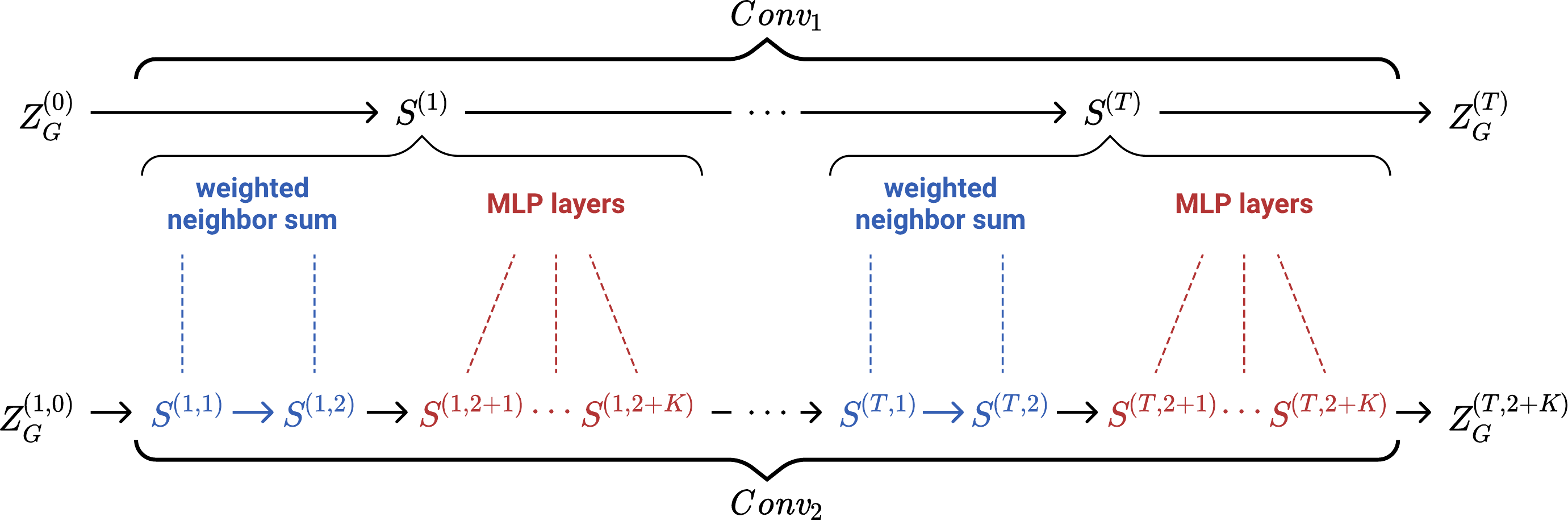}
		\caption[Illustration of a 2-\acs{wl}-\acs{gnn} architecture that simulates vertex neighborhood convolutions.]{
			Illustration of the correspondence between $\mathit{Conv}_1$ and $\mathit{Conv}_2$.
		}\label{fig:wl2conv:wl2-simulation}
	\end{figure}

	\noindent Let $\varphi: \mathbb{R}^{d^{(T, 2 + K)}} \to \mathbb{R}^{d^{(T)}} \cup \{ \mathtt{nil} \}$ be a function that maps the final 2-\acs{wl} feature vectors produced by $\mathit{Conv}_2$ to the output space of $\mathit{Conv}_1$ or the constant \texttt{nil}.
	Let $\mathit{Pool}_2\left(Z_G^{(T,2+K)}\right) \coloneqq \mathit{Pool}_1(\ldblbrace z_{ij}\, |\, z_{ij} = \varphi\left(Z_G^{(T,2+K)}[e_{ij}]\right) \land e_{ij} \in \mathcal{E}_{G^1} \land z_{ij} \neq \mathtt{nil} \rdblbrace)$.
	\Cref{thm:wl2conv:wl2-simulation} then follows if there is a function $\varphi$ s.t.\ $\forall v_i \in \mathcal{V}_G: \mathit{Conv}_1(G)[v_i] = \varphi(\mathit{Conv}_2(G)[e_{ii}])$ and $\forall e_{ij} \in \mathcal{E}_G: \varphi(\mathit{Conv}_2(G)[e_{ij}]) = \mathtt{nil}$.
	To guarantee that there is such a function $\varphi$, we now inductively prove the following three invariants, which have to hold for all $t \in \{ 0, \dots, T \}$:
	\begin{enumerate}[label={(P\arabic*)}]
		\item\label[inv]{inv:wl2conv:wl2-simulation:indicator}
			${Z_G^{(t,2+K)}[e_{ij}]}_1 = \mathbbm{1}[i = j]$, i.e., the first component of each 2-\acs{wl} feature vector allows $\varphi$ to decide whether that vector should be mapped to \texttt{nil}.
		\item\label[inv]{inv:wl2conv:wl2-simulation:feature}
			${Z_G^{(t,2+K)}[e_{ii}]}_{2, \dots, (d^{(t)} + 1)} = Z_G^{(t)}[v_{i}]$, i.e., the second to $(1 + d^{(t)})$-th components of each self-loop feature vector in $h_2$ contain the corresponding convolved vertex feature vector at layer $t$ in $h_1$.
		\item\label[inv]{inv:wl2conv:wl2-simulation:weight}
			${Z_G^{(t,2+K)}[e_{ij}]}_{d^{(t)} + 2} = w_{ij}$, i.e., the weights for the vertex neighborhood sums are encoded in the edge and self-loop feature vectors.
	\end{enumerate}

	For $t = 0$, all invariants apply to the initial feature matrix $Z_G^{(0, 2+K)} = Z_G^{(1, 0)}$ by \cref{defn:wl2conv:wl2-conv-init}:
	\begin{align*}
		\forall v_i \in \mathcal{V}_G: Z_G^{(1, 0)}[e_{ii}] \coloneqq (1) \oplus x[v_i] \oplus (w_{ii})
		\text{ and }
		\forall e_{ij} \in \mathcal{E}_G: Z_G^{(1, 0)}[e_{ij}] \coloneqq (0) \oplus \mathbf{0} \oplus (w_{ij})
		\text{.}
	\end{align*}
	Assuming the invariants hold for $t - 1$, we now show that they also hold for $t$.
	The layers $S^{(t, 1)}$ and $S^{(t, 2)}$ are used to compute the weighted vertex neighborhood sums
	\begin{align*}
		{Z^{(t, 2)}[e_{ii}]}_{2, \dots, (1+d^{(t-1)})} = w_{ii} {Z^{(t, 0)}[e_{ii}]}_{2, \dots, (1+d^{(t-1)})} + \smashoperator[lr]{\sum_{v_j \in \Gamma_G(v_i)}} w_{ij} {Z^{(t, 0)}[e_{jj}]}_{2, \dots, (1+d^{(t-1)})}
		\text{.}
	\end{align*}
	We now explicitly define parameter matrices for $S^{(t, 1)}$ and $S^{(t, 2)}$ s.t.\ this weighted sum is produced.
	Note that the weighted vertex neighborhood sum only requires scalar multiplication and vector addition, i.e., the $d^{(t-1)}$ vertex feature dimensions are mutually independent.
	W.l.o.g.\ this allows us to simplify notation by treating the vertex feature vectors as if they were scalars in the following definitions, i.e., we can assume $d^{(t-1)} = 1$ and $Z^{(t,0)}[e_{ii}] = (1, Z^{(t-1)}[v_i], w_{ii}) \in \mathbb{R}^{3}$.
	Using this simplification, the layer $S^{(t, 1)}$ is defined by
	\begin{gather*}\allowdisplaybreaks
		\begin{aligned}
			Z^{(t,1)}[e_{ij}] &= \textcolor{t_blue}{Z^{(t,0)}[e_{ij}] W_{\mathrm{L}}^{(t,1)}} + \smashoperator[lr]{\sum_{v_l \in {\Gamma_{G^1}(v_i) \cap \Gamma_{G^1}(v_j)}}} \textcolor{t_red}{\left( Z^{(t,0)}[e_{ij}] W_{\mathrm{F}}^{(t,1)}\right)} \odot \textcolor{t_darkgreen}{\left( \left( Z^{(t,0)}[e_{il}] + Z^{(t,0)}[e_{lj}] \right) W_{\Gamma}^{(t,1)} \right)} \\ 
			&= \begin{cases}
				\textcolor{t_blue}{(1, 0, w_{ii}, 0)} + \textcolor{t_darkgreen}{\left(0, 0, 0, 2\textcolor{t_red}{w_{ii}} Z^{(t-1)}[v_i]\right)} + \smashoperator[lr]{\sum_{v_l \in \Gamma_{G^1}(v_i)}} \textcolor{t_darkgreen}{\left(0, \tfrac{2}{\textcolor{t_red}{2}} \textcolor{t_red}{Z^{(t-1)}[v_i]} w_{il}, 0, 0\right)} & \text{if $i = j$} \\ 
				\textcolor{t_blue}{(0, 0, w_{ij}, 0)} + \textcolor{t_darkgreen}{\left(0, 0, 0, \textcolor{t_red}{w_{ij}} (Z^{(t-1)}[v_i] + 0)\right)} + \textcolor{t_darkgreen}{\left(0, 0, 0, \textcolor{t_red}{w_{ij}} (0 + Z^{(t-1)}[v_j])\right)} & \text{else} 
			\end{cases} \\
			&= \begin{cases}
				\left(1, \mkern6mu\smashoperator[lr]{\sum_{v_l \in \Gamma_{G^1}(v_i)}} w_{il} Z^{(t-1)}[v_i], w_{ii}, 2 w_{ii} Z^{(t-1)}[v_i] \right) & \text{if $i = j$} \\
				\left(0, 0, w_{ij}, w_{ij} (Z^{(t-1)}[v_i] + Z^{(t-1)}[v_j])\right) & \text{else}
			\end{cases} \enspace ,
		\end{aligned}\\[10pt]
		\text{with }
		W_{\mathrm{L}}^{(t,1)} \coloneqq \begin{pmatrix}
			\textcolor{t_blue}{1} & 0 & 0 & 0 \\ 
			0 & 0 & 0 & 0 \\
			0 & 0 & \textcolor{t_blue}{1} & 0 
		\end{pmatrix},
		W_{\mathrm{F}}^{(t,1)} \coloneqq \begin{pmatrix}
			0 & 0 & 0 & 0 \\
			0 & \textcolor{t_red}{\frac{1}{2}} & 0 & 0 \\ 
			0 & 0 & 0 & \textcolor{t_red}{1} 
		\end{pmatrix},
		W_{\Gamma}^{(t,1)} \coloneqq \begin{pmatrix}
			0 & 0 & 0 & 0 \\
			0 & 0 & 0 & \textcolor{t_darkgreen}{1} \\ 
			0 & \textcolor{t_darkgreen}{1} & 0 & 0 
		\end{pmatrix}
		\text{.}
	\end{gather*}
	The vertex neighborhood summation is completed via $S^{(t, 2)}$, which is defined by
	\begin{gather*}\allowdisplaybreaks
		\begin{aligned}
			Z^{(t,2)}[e_{ij}] &= \begin{cases}
				\textcolor{t_blue}{\left(1, -\smashoperator[lr]{\sum_{v_l \in \Gamma_{G^1}(v_i)}} w_{il} Z^{(t-1)}[v_i], w_{ii} \right)} + \smashoperator[lr]{\sum_{v_l \in \Gamma_{G^1}(v_i)}} \textcolor{t_darkgreen}{\left(0, w_{il}\left(Z^{(t-1)}[v_i] + Z^{(t-1)}[v_l]\right), 0\right)} & \text{if $i = j$} \\ 
				\textcolor{t_blue}{(0, 0, w_{ij})} & \text{else} 
			\end{cases} \\
			&= \begin{cases}
				\left(1, w_{ii} Z^{(t-1)}[v_i] + \smashoperator[lr]{\sum_{v_l \in \Gamma_{G}(v_i)}} w_{il} Z^{(t-1)}[v_l], w_{ii}\right) & \text{if $i = j$} \\
				(0, 0, w_{ij}) & \text{else}
			\end{cases} \enspace ,
		\end{aligned}\\
		\text{with }
		W_{\mathrm{L}}^{(t,2)} \coloneqq \begin{pmatrix}
			\textcolor{t_blue}{1} & 0 & 0 \\ 
			0 & \textcolor{t_blue}{-1} & 0 \\ 
			0 & 0 & \textcolor{t_blue}{1} \\ 
			0 & 0 & 0
		\end{pmatrix},\quad
		W_{\mathrm{F}}^{(t,2)} \coloneqq \begin{pmatrix}
			0 & \textcolor{t_red}{\frac{1}{2}} & 0 \\ 
			0 & 0 & 0 \\
			0 & 0 & 0 \\
			0 & 0 & 0
		\end{pmatrix},\quad
		W_{\Gamma}^{(t,2)} \coloneqq \begin{pmatrix}
			0 & 0 & 0 \\
			0 & 0 & 0 \\
			0 & 0 & 0 \\
			0 & \textcolor{t_darkgreen}{1} & 0 
		\end{pmatrix}
		\text{.}
	\end{gather*}
	\begin{figure}[ht]
		\centering
		\includegraphics[width=\linewidth]{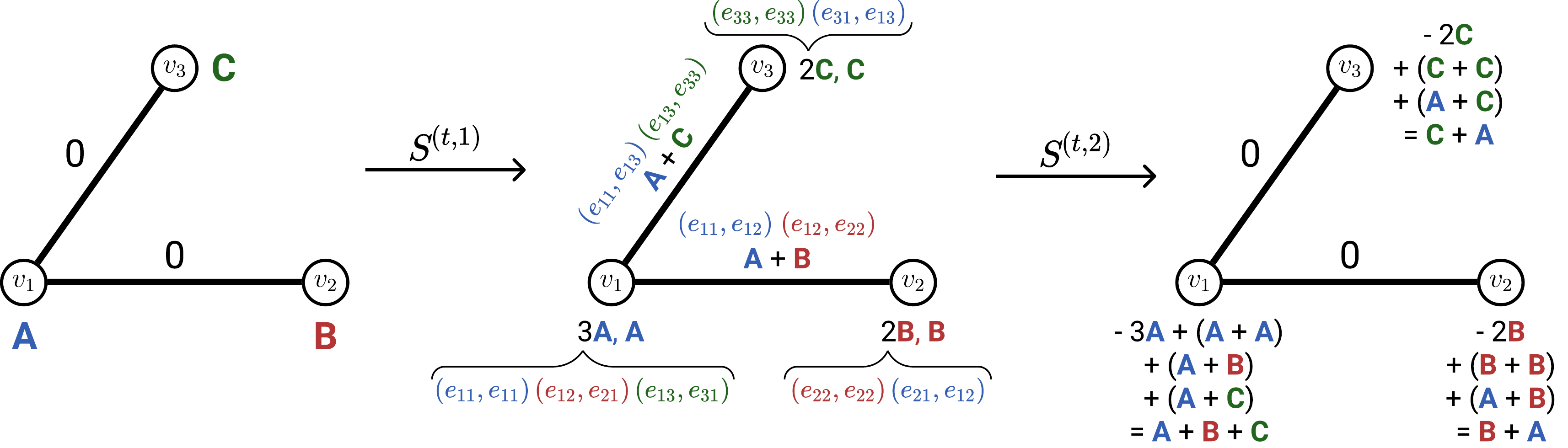}
		\caption[Intuition for how $S^{(t, 1)}$ and $S^{(t, 2)}$ compute vertex neighborhood sums.]{
			Intuition for how $S^{(t, 1)}$ and $S^{(t, 2)}$ compute vertex neighborhood sums in two steps.
			For simplicity, weights are ignored, i.e.\ all $w_{ij} = 1$.
			For each self-loop/edge $e_{ij} \in \mathcal{E}_{G^1}$, the set of localized 2-\acs{wl} neighbors $\Gamma_{G^1}(v_i) \cap \Gamma_{G^1}(v_j)$ is shown in the middle, after the first step.
			The colors in this illustration are unrelated to the colored parts in the previous equations.
		}\label{fig:wl2conv:wl2-vert-agg-intuition}
	\end{figure}
	
	\noindent Using the two layers $S^{(t, 1)}$ and $S^{(t, 2)}$ that we just defined, the weighted vertex neighborhood sum for all $v_i \in \mathcal{V}_G$ is contained in $Z^{(t,2)}[e_{ii}]$.
	Additionally, for all $e_{ij} \in \mathcal{E}_{G^1}$, the indicators ${Z^{(t,2)}[e_{ij}]}_1 = \mathbbm{1}[i = j]$ and the weights ${Z^{(t,2)}[e_{ij}]}_{d^{(t)} + 2} = w_{ij}$ are preserved.
	This means that \cref{inv:wl2conv:wl2-simulation:indicator,inv:wl2conv:wl2-simulation:weight} are satisfied after $S^{(t, 2)}$.

	To complete the induction step, it now remains to show that all three invariants hold after applying the layers $S^{(t, 2+1)}, \dots, S^{(t, 2+K)}$.
	Note that a 2-\acs{wl} convolution layer is reduced to a fully connected layer if $W_{\mathrm{F}}^{(t)} = \mathbf{0}$.
	Via the universal approximation theorem~\citep{Hornik1991}, we can therefore use $S^{(t, 2+1)}, \dots, S^{(t, 2+K)}$ to simulate the $K$ layers of $\mathrm{\acs*{mlp}}^{(t)}$ without changing the first and last dimension of each feature vector to preserve \cref{inv:wl2conv:wl2-simulation:indicator,inv:wl2conv:wl2-simulation:weight}.
	The resulting feature matrix $Z^{(t,2+K)}$ then satisfies all three invariants, which completes the induction.

	Using \cref{inv:wl2conv:wl2-simulation:indicator,inv:wl2conv:wl2-simulation:feature} for $t = T$, we can therefore set
	\begin{align*}
		\varphi\left(Z_G^{(T,2+K)}[e_{ij}]\right) \coloneqq \begin{cases}
			{Z_G^{(T,2+K)}[e_{ij}]}_{2, \dots, (d^{T} + 1)} & \text{if ${Z_G^{(T,2+K)}[e_{ij}]}_1 = 1$} \\
			\mathtt{nil} & \text{else}
		\end{cases}
		\text{.}
	\end{align*}
	By our previous definition of $\mathit{Pool}_2$, this in turn implies that $\mathit{Pool}_2(Z_G^{(T,2+K)}) = \mathit{Pool}_1(Z_G^{(T)}) \allowbreak\iff h_2(G) = h_1(G)$, which concludes the proof.
\end{proof}
\begin{cor}\label{cor:wl2conv:wl2-gnn-wl1-power}
	2-\acs{wl}-\acsp{gnn} have at least the same \ac{dp} as 1-\acs{wl}.
\end{cor}
\begin{proof}
	The corollary directly follows from the fact that 2-\acs{wl}-\acsp{gnn} can simulate \acp{gin} by \cref{thm:wl2conv:wl2-simulation} and the fact that \acp{gin} have the same \ac{dp} as 1-\acs{wl} by \cref{prop:prelim:gcnn-wl1-limit}, because they use injective vertex neighborhood hashing functions.
\end{proof}

To complete our analysis of the \ac{dp} of the 2-\acs{wl}-\acs{gnn}, we now show that it is not just as discriminative as 1-\acs{wl} but is in fact more discriminative than 1-\acs{wl}.
\begin{prop}\label{cor:wl2conv:wl2-gnn-regular}
	There are $d$-regular graphs $G$ and $H$ of size $n$, which can be distinguished by 2-\acs{wl}-\acsp{gnn}.
\end{prop}
\begin{proof}
	The proposition follows if we choose the six-cycle graph for $G$ and the two three-cycles graph for $H$ (see \cref{fig:prelim:wl1-problem}).
	Let $h_2 = \mathit{Pool} \circ S$ be a 2-\acs{wl}-\acs{gnn} with the neighborhood radius $r = 1$, which consists of a single 2-\acs{wl} convolution layer $S: \mathbb{R}^{* \times 2} \to \mathbb{R}^{* \times 1}$ and the pooling layer $\mathit{Pool} = \min$.
	In accordance with \cref{defn:wl2conv:wl2-conv-init}, we set the initial feature vectors of the vertices $v_i$ of $G$ and $H$ to $Z^{(0)}[e_{ii}] \coloneqq (1, 0)$ and the initial feature vectors of their edges $e_{ij}$ to $Z^{(0)}[e_{ij}] \coloneqq (0, 1)$.
	Let the weight matrices of $S$ be $W_{\mathrm{L}} \coloneqq \mathbf{0}$ and $W_{\mathrm{F}} = W_{\Gamma} \coloneqq \begin{pmatrix} 1\\1 \end{pmatrix}$.
	For simplicity, we choose the identity activation functions $\sigma = \sigma_{\Gamma} = \mathrm{id}$.
	By \cref{defn:wl2conv:wl2-conv-step}, all self-loops $e_{ii}$ of $G^1$ and $H^1$ have the three neighbors $\{ \textcolor{t_blue}{\ldblbrace e_{ii}, e_{ii} \rdblbrace},\allowbreak \textcolor{t_red}{\ldblbrace e_{ij}, e_{ji} \rdblbrace},\allowbreak \textcolor{t_darkgreen}{\ldblbrace e_{il}, e_{li} \rdblbrace} \}$, i.e., the length-two walk along $e_{ii}$ itself and the length-two walks to and from the two neighboring vertices $\Gamma(v_i) = \{ v_j, v_l \}$. 
	Therefore, the convolved feature vector of all self-loops are $Z^{(1)}[e_{ii}] = (1) \odot (\textcolor{t_blue}{(1 + 1)} + \textcolor{t_red}{(1 + 1)} + \textcolor{t_darkgreen}{(1 + 1)}) = 6$. 
	However, for the non-self-loops of $G^1$ and $H^1$, i.e., the edges of $G$ and $H$, we get differing convolved feature vectors.
	The 2-\acs{wl} neighbors of $e_{ij} \in \mathcal{E}_G$ are $\{ \ldblbrace e_{ii},  e_{ij} \rdblbrace,\allowbreak \ldblbrace e_{ij}, e_{jj} \rdblbrace \}$.
	The 2-\acs{wl} neighbors of $e_{ij}' \in \mathcal{E}_H$ are $\{ \ldblbrace e_{ii}',  e_{ij}' \rdblbrace,\allowbreak \ldblbrace e_{ij}',  e_{jj}' \rdblbrace,\allowbreak \ldblbrace e_{il}',  e_{lj}' \rdblbrace \}$, where $v_l' \in \mathcal{V}_H$ is the common neighbor of $v_i'$ and $v_j'$.
	The different neighborhood sizes of the edges of $G$ and $H$ imply that $\forall e_{ij} \in \mathcal{E}_G: Z^{(1)}[e_{ij}] = 4$, while $\forall e_{ij}' \in \mathcal{E}_H: Z^{(1)}[e_{ij}] = 6$.
	Thus $h_2(G) = \min\{ 4, 6 \} \neq \min\{ 6, 6 \} = h_2(H)$, which concludes the proof. 
\end{proof}
\begin{cor}\label{cor:wl2conv:wl2-gnn-more-wl1-power}
	The \ac{dp} of 2-\acs{wl}-\acsp{gnn} is strictly higher than that of the 1-\acs{wl} algorithm.
\end{cor}
\begin{proof}
	The corollary directly follows from \cref{cor:wl2conv:wl2-gnn-wl1-power} and \cref{cor:wl2conv:wl2-gnn-regular}, since 1-\acs{wl} cannot distinguish regular graphs~\citep[Cor.~1.8.5]{Immerman1990}.
\end{proof}

This concludes our analysis of the \acl{dp} of 2-\acs{wl}-\acsp{gnn}.
The key insight in this section is that 2-\acs{wl}-\acsp{gnn} are more discriminative than all vertex neighborhood aggregation \acp{gnn}, because the \ac{dp} of the latter is at most that of 1-\acs{wl}.
Additionally, we can conclude that 2-\acs{wl}-\acsp{gnn} are able to distinguish graphs that are indistinguishable by 2-\acsp{gnn} due to \cref{prop:wl2conv:2gnn-regular-limit,cor:wl2conv:wl2-gnn-regular}.

Note that no statement regarding the \ac{dp} of 2-\acs{wl}-\acsp{gnn} compared to 2-\acs{wl} was made.
It is easy to see that 2-\acs{wl}-\acsp{gnn} \textit{generally} cannot have the same power as 2-\acs{wl} due to the neighborhood localization simplification:
For a small neighborhood radius of $r = 1$, nonexistent edges $e_{ij} \notin \mathcal{E}_G$ do not have a feature vector;
those missing feature vectors are however required by the proof of 2-\acs{wl}'s $m$-cycle counting ability for $m \geq 4$ \citep[see the proof by][Lem.~1 and Thm.~2]{Fuerer2017}.
We leave a thorough discussion of the relation between 2-\acs{wl}-\acsp{gnn} and 2-\acs{wl} for future work.

\section{Evaluation}%
\label{sec:eval}

In our experimental evaluation, we compare the proposed 2-\acs{wl} convolution layer with other state-of-the-art approaches. We focus on two types of learners:
\acp{svm} using \aclp{gk} and \acp{gnn}.
We evaluate those learners by comparing their test accuracies on multiple binary classification problems.
To obtain those accuracies, we use the graph classification benchmarking framework recently proposed by \citet{Errica2020}:
We use 10-fold stratified training/test splits; for each split the hyperparameters are tuned via a second 90\%/10\% validation holdout split of the training data.
Experiments were run three times to smooth out differences caused by random weight initialization.
To train models that require gradient-based optimization, we use the well-known Adam optimizer and the standard binary cross-entropy loss. 

Using this assessment strategy, we evaluate \acp{svm} with the following \aclp{gk}:
\ac{wlst}, \ac{wlsp}, and the so-called 2-LWL and 2-GWL kernels~\citep{Morris2017};
the last two are essentially 2-\acs{wl} variants of the \ac{wlst} kernel.
We additionally evaluate the following \acp{gnn}:
2-\acs{wl}-\acs{gnn} \textit{(our method)},
\ac{gin},
2-\acs{gnn}, and a structure-unaware baseline that applies an \ac{mlp} to each vertex feature vector $x_G[v]$, sums up the resulting vectors and applies another \ac{mlp} to the sum \citep[see][]{Errica2020}.

\subsection{Synthetic Data}%
\label{sec:eval:synthetic}

We begin with an evaluation 
on a synthetic binary classification dataset, which demonstrates the potential advantages of a higher dimensional \ac{wl} method. 
To determine the classes of the graphs in this dataset, a learner has to solve the following \textit{unicolored triangle detection} problem:
Given a graph $G$ with vertices that are colored as either $l_G[v] = \colorlabel{t_blue}{A}$ or as $l_G[v] = \colorlabel{t_red}{B}$, the learner has to find the unique triangle $(v_i, v_j, v_k)$ in $G$ for which $l_G[v_i] = l_G[v_j] = l_G[v_k]$. 
The class of $G$ is then determined by the color of the vertices $(v_i, v_j, v_k)$.

Based on this problem, we generated a synthetic triangle detection dataset.
It contains randomly generated graphs with varying vertex counts and vertex color proportions.
We use this dataset to evaluate whether a learner is able to ignore varying amounts of noisy random structure and focus on relevant local substructures, in this case unicolored triangles.
For the evaluation of 2-\acs{wl}-\acsp{gnn} the neighborhood radius $r = 2$ is used.
\begin{table}[ht]
\small
	\caption{Mean accuracies and standard deviations on the triangle detection dataset.}\label{tbl:eval:synthetic}
	\centering
	\begin{filecontents*}{data/results.csv}
id;model;params;isDefault;isLta;triangleTestMean;triangleTestStd;triangleTrainMean;triangleTrainStd;triangleBestTest;triangleBestTrain;nciTestMean;nciTestStd;nciTrainMean;nciTrainStd;nciBestTest;nciBestTrain;proteinsTestMean;proteinsTestStd;proteinsTrainMean;proteinsTrainStd;proteinsBestTest;proteinsBestTrain;ddTestMean;ddTestStd;ddTrainMean;ddTrainStd;ddBestTest;ddBestTrain;redditTestMean;redditTestStd;redditTrainMean;redditTrainStd;redditBestTest;redditBestTrain;imdbTestMean;imdbTestStd;imdbTrainMean;imdbTrainStd;imdbBestTest;imdbBestTrain
0;WL\textsubscript{ST};T=1;1;1;64.8;13.2;88.3;6.9;0;0;73.9;2.6;76.7;1.2;0;0;72.8;3.3;76.5;3.6;0;0;78.9;4.2;94.8;8.0;1;0;76.3;2.5;81.7;3.6;0;0;71.0;2.2;71.7;1.3;0;0
1;WL\textsubscript{ST};T=2;0;1;59.1;13.1;93.0;6.2;0;0;81.3;2.2;88.3;3.0;0;0;74.0;3.6;86.1;10.8;1;0;78.3;4.4;96.8;4.9;0;0;80.5;1.4;95.0;3.0;1;0;70.4;2.9;79.0;6.5;0;0
2;WL\textsubscript{ST};T=3;1;1;56.9;11.1;98.0;1.7;0;0;84.8;1.6;97.6;2.6;1;0;73.0;2.4;79.5;10.4;0;0;78.8;4.3;95.1;6.3;0;0;78.0;2.7;94.4;1.9;1;0;72.9;2.5;82.0;6.2;0;0
3;WL\textsubscript{ST};T=4;0;1;60.4;10.7;99.7;0.5;0;0;85.0;1.8;97.3;3.0;0;0;73.0;2.6;79.4;9.5;0;0;78.5;4.3;95.1;7.8;1;0;77.7;2.7;95.5;1.9;0;0;72.6;2.4;83.4;6.7;0;0
4;WL\textsubscript{ST};T=5;0;1;62.6;11.2;100.0;0.0;0;1;85.4;2.0;98.2;1.7;1;1;73.3;2.2;82.0;11.5;0;0;77.8;4.3;97.7;5.8;0;1;77.3;3.7;96.4;1.6;0;1;72.7;2.2;84.5;6.1;0;0
5;WL\textsubscript{SP};T=3;1;0;68.0;10.7;96.9;8.4;1;0;m;m;m;m;0;0;73.1;3.5;88.5;6.9;0;0;m;m;m;m;0;0;m;m;m;m;0;0;74.4;3.5;89.0;0.7;1;1
6;WL\textsubscript{SP};T=5;0;0;69.7;9.5;100.0;0.0;1;1;m;m;m;m;0;0;73.4;2.7;92.3;5.8;0;1;m;m;m;m;0;0;m;m;m;m;0;0;73.6;2.6;89.1;0.6;1;1
7;2-LWL;T=3;1;1;56.5;6.2;97.3;3.3;0;0;76.7;2.2;85.0;2.1;0;0;69.4;4.6;85.1;6.8;0;0;76.6;3.5;100.0;0.0;0;1;75.8;2.9;90.5;0.8;0;0;72.2;3.3;73.2;0.9;0;0
8;2-GWL;T=3;1;0;61.8;8.8;99.9;0.2;0;1;71.6;2.1;99.7;0.3;0;1;73.1;3.6;89.8;6.0;1;1;76.3;3.9;100.0;0.0;0;1;75.4;3.2;99.4;0.6;0;1;70.4;3.2;83.7;7.3;0;0
9;Baseline;\mathrm{sum};1;0;44.6;8.1;48.8;1.6;0;0;67.7;3.1;71.7;4.0;0;0;74.0;4.9;73.9;3.8;0;0;75.7;2.5;89.7;6.1;1;1;72.1;7.8;0.0;0.0;0;0;50.7;2.4;0.0;0.0;0;0
10;GIN;\mathrm{sum};1;0;70.0;7.4;84.2;10.6;0;0;77.4;2.9;79.5;4.2;0;0;71.8;3.1;76.4;3.2;0;0;75.2;3.4;83.6;5.1;0;0;87.0;4.4;0.0;0.0;0;0;66.8;3.9;0.0;0.0;0;0
11;2-GNN;\mean;1;0;76.8;10.7;93.2;3.1;0;0;75.9;2.0;84.3;2.0;0;0;74.8;3.4;81.3;3.7;0;0;72.9;4.1;82.9;4.3;0;0;m;m;m;m;0;0;71.4;3.6;70.7;1.8;0;0
12;2-GNN;\wmean;1;0;81.8;7.6;97.1;2.9;0;0;78.3;1.8;86.1;2.2;1;1;73.8;3.5;84.0;6.0;0;1;69.6;3.9;83.5;6.9;0;0;m;m;m;m;0;0;70.9;3.2;72.2;1.2;0;0
14;2-WL-GNN;\mean;0;1;91.1;4.8;97.3;2.3;0;0;65.8;3.4;67.9;1.5;1;1;74.6;3.4;79.1;2.2;0;0;70.6;3.6;73.1;2.7;1;0;t;t;t;t;0;0;69.9;4.4;70.6;3.4;0;0
15;2-WL-GNN;\mean;0;0;95.8;4.6;98.5;2.5;0;0;t;t;t;t;0;0;t;t;t;t;0;0;t;t;t;t;0;0;t;t;t;t;0;0;t;t;t;t;0;0
19;2-WL-GNN;\mean;1;0;92.9;8.4;98.3;2.6;0;0;72.4;2.9;77.0;1.7;0;0;76.5;2.7;81.5;3.2;1;0;75.4;3.3;82.3;2.5;0;0;83.7;5.2;85.5;6.4;0;0;71.2;4.0;72.7;3.0;0;0
21;2-WL-GNN;\wmean;1;0;99.4;1.3;99.8;0.4;1;1;73.5;2.9;77.4;2.1;0;0;76.2;3.3;78.9;1.7;0;0;74.7;3.1;86.3;6.2;0;0;89.4;2.6;91.6;1.5;1;1;72.2;3.1;72.2;1.6;1;0
\end{filecontents*}
\csvreader[
		column count=40,
		tabular={clrr},
		separator=semicolon,
		table head={%
			& \multicolumn{1}{l}{Model (Iterations/Pooling)} & \multicolumn{1}{c}{Train} & \multicolumn{1}{c}{Test} \\\toprule%
		},
		table foot=\bottomrule,
		late after line=\ifthenelse{\equal{\id}{9}}{\\\midrule}{\\},
		head to column names,
		filter=\equal{\isDefault}{1}\or\equal{\id}{4}
	]{data/results.csv}{}{%
		\ifthenelse{\equal{\id}{0}}{\multirow{6}{*}[0em]{\rotatebox[origin=c]{90}{\small\textsc{Kernel}}}}{}%
		\ifthenelse{\equal{\id}{9}}{\multirow{6}{*}[0em]{\rotatebox[origin=c]{90}{\small\textsc{\ac{gnn}}}}}{} &%
		\textbf{\model}~($\params$) &%
		\evalres[\and\not\equal{\id}{8}]{\triangleBestTrain}{\triangleTrainMean}{\triangleTrainStd} &%
		\evalres{\triangleBestTest}{\triangleTestMean}{\triangleTestStd}%
	}
	\hspace*{2em}%
	\raisebox{-.5\height}{\includegraphics[width=0.18\linewidth]{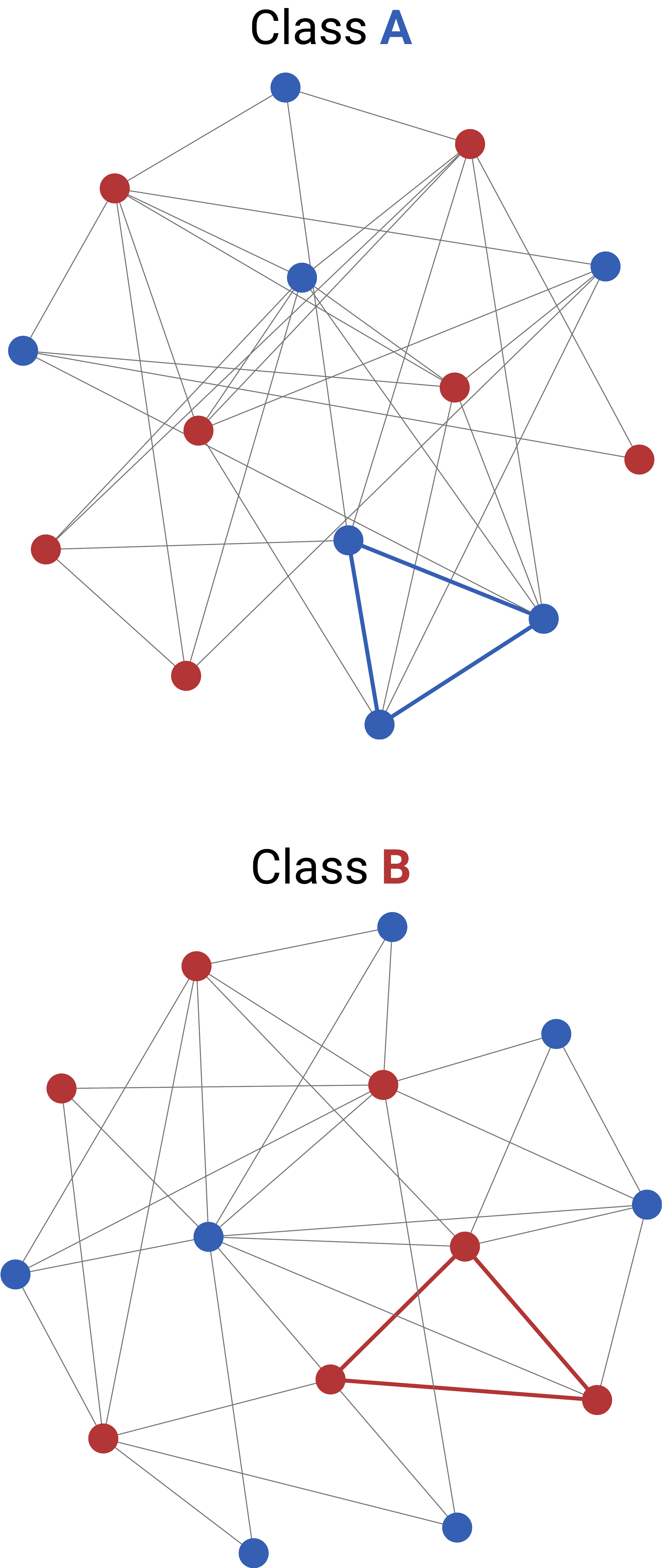}}
\end{table}

\noindent Looking at the results in \cref{tbl:eval:synthetic}, it can be seen that the structure-unaware baseline method is completely unable to detect triangles, as expected.
The structure-aware learners on the other hand all perform better than random guessing and are in fact mostly able to fit the training data perfectly.
This shows that all generated graphs are 1-\acs{wl} distinguishable; the \ac{wl} subtree kernel \ac{svm}, for example, can simply ``memorize'' the training graphs via their unique 1-\acs{wl} color distribution after $T = 5$ refinement steps.

However, the ability to distinguish training graphs is not sufficient to also classify previously unseen graphs correctly.
Since 1-\acs{wl} cannot detect triangles, all 1-\acs{wl} bounded approaches (\acs{wl}\textsubscript{ST}, \acs{wl}\textsubscript{SP}, Baseline, \acs{gin}) are therefore unable to generalize, as suggested by their test accuracies.
Performance better than random guessing can be explained by availability of the following proxy indicator:
The presence of an \colorlabel{t_blue}{A}-colored triangle in a graph $G$ implies that there is a local region with a slightly higher density of \colorlabel{t_blue}{A}-colored vertices than in a \colorlabel{t_red}{B}-colored graph $H$ with the same vertex color proportions.
This local difference in color density is already detectable in the depth-1 \acs{bfs} subtrees used by 1-\acs{wl} after a single refinement step, which explains why \acs{wl}\textsubscript{ST} performs similarly for $T = 1$ and $T = 5$.

As for the 2-\acs{wl} inspired kernels, 2-LWL and 2-GWL, it is
interesting to see that both kernels do not appear to generalize better than the 1-\acs{wl} bounded methods.
We explain this by the fairly small size of the triangle detection dataset (228 graphs);
even though both kernels embed graphs into a space which contains dimensions that indicate the presence of a unicolored triangle, i.e., their \ac{dp} is sufficiently high to solve the problem, there are so many of those triangle-indicating embedding dimensions that the relevant indicator dimensions found in a given training split might not overlap with those in the test split.

Looking at the 2-\acs{wl} inspired \acp{gnn} (2-\acs{gnn}, 2-\acs{wl}-\acs{gnn}), we find that the 2-\acs{wl}-\acs{gnn} significantly outperforms all other methods, which is in line with our results from \cref{sec:wl2conv:limit,sec:wl2conv}.

\subsection{Evaluation on Real-World Data}%
\label{sec:eval:real}


We evaluate the approaches on five common binary graph classification benchmark datasets, namely the NCI1~\citep{Shervashidze2011}, PROTEINS~\citep{Borgwardt2005a}, and D\&D~\citep{Dobson2003} datasets from the domain of bioinformatics, and the REDDIT-B and IMDB-B datasets~\citep{Yanardag2015} from social network analysis.
\Cref{tbl:eval:real} shows our evaluation results.
For the evaluation of 2-\acs{wl}-\acsp{gnn}, different neighborhood radii were used for each dataset.
In the order of the columns in the table, the results were obtained with the radii $r = 8$, $5$, $2$, $1$ and $4$, respectively.
\begin{table}[ht]
\footnotesize
    \def\wmean{\mathrm{w.\ mean}}
	\caption[Mean test accuracies and standard deviations on real-world data.]{
		Mean test accuracies and standard deviations on real-world data.
	}\label{tbl:eval:real}
	\centering
	\csvreader[
		tabular={clrrrrr},
		separator=semicolon,
		before reading=\setlength{\tabcolsep}{4pt},
		table head={%
			&%
			\multicolumn{1}{c}{Model (Iter./Pooling)} &%
			\multicolumn{1}{c}{\textbf{NCI1}} &%
			\multicolumn{1}{c}{\textbf{PROTEINS}} &%
			\multicolumn{1}{c}{\textbf{D\&D}} &%
			\multicolumn{1}{c}{\textbf{REDDIT-B}} &%
			\multicolumn{1}{c}{\textbf{IMDB-B}}%
			\\\toprule%
		},
		table foot=\bottomrule,
		late after line=\ifthenelse{\equal{\id}{9}}{\\\midrule}{\\},
		head to column names,
		filter={\equal{\isDefault}{1}}
	]{data/results.csv}{}{%
		\ifthenelse{\equal{\id}{0}}{\multirow{5}{*}[0em]{\rotatebox[origin=c]{90}{\small\textsc{Kernel}}}}{}%
		\ifthenelse{\equal{\id}{9}}{\multirow{6}{*}[0em]{\rotatebox[origin=c]{90}{\small\textsc{\ac{gnn}}}}}{} &%
		\textbf{\model}~($\params$) &%
		{\small\evalres{\nciBestTest}{\nciTestMean}{\nciTestStd}}&%
		{\small\evalres{\proteinsBestTest}{\proteinsTestMean}{\proteinsTestStd}}&%
		{\small\evalres{\ddBestTest}{\ddTestMean}{\ddTestStd}}&%
		{\small\evalres{\redditBestTest}{\redditTestMean}{\redditTestStd}}&%
		{\small\evalres{\imdbBestTest}{\imdbTestMean}{\imdbTestStd}}%
	}
\end{table}

\noindent Compared with the triangle detection dataset, the advantage of 2-\acs{wl}-\acsp{gnn} over the other approaches is clearly less pronounced.
This indicates that the theoretical advantages of 2-\acs{wl} over 1-\acs{wl} are not as relevant for the five real-world domains as they are for the synthetic problem.
Nonetheless, the test performance of 2-\acs{wl}-\acsp{gnn} is generally comparable to that of the other state-of-the-art learners, in the sense that the performance of the evaluated 2-\acs{wl}-\acs{gnn} models is within the $2\sigma$ confidence interval of the best evaluated model.

If we look at the enzyme detection problem (PROTEINS and D\&D), we observe that all evaluated approaches appear to be unable to leverage structural information for a significant improvement over the baseline learner.
On the social network datasets (REDDIT-B and IMDB-B), on the other hand, the structure aware methods clearly outperform the baseline.
This confirms similar results by \citet{Errica2020}.

\section{Conclusion}%
\label{sec:conclusion}

We proposed the novel 2-\acs{wl}-\acs{gnn} and showed it to be strictly more discriminative than 1-\acs{wl} bounded \acp{gnn}.
This theoretical advantage was clearly confirmed experimentally on synthetic data, while results competitive to state-of-the-art \aclp{gk} and \acp{gnn} could be achieved on real-world data. We envision two main directions for future research:
First, a more thorough theoretical analysis of the relation between 2-\acs{wl}-\acsp{gnn} and 2-\acs{wl} is required to answer questions such as how the neighborhood radius $r$ relates to the \acl{dp} of a 2-\acs{wl}-\acs{gnn}.
Second, evaluations on a broader range of domains and other problem types, such as vertex labeling, link prediction, or graph regression, will help to determine in which contexts the theoretical advantages of 2-\acs{wl}-\acsp{gnn} also lead to practical improvements.

\long\def\acks#1{\section*{Acknowledgments}#1}
\acks{This work was supported by German Research Foundation (DFG) within the Collaborative Research Center ``On-The-Fly Computing'' (SFB 901/3 project no.\ 160364472).}

\vfill\pagebreak
\appendix
\section{Implementation of 2-\acs*{wl}-\acsp*{gnn} on \acsp*{gpgpu}}%
\label{sec:appendix:implementation}

In the main paper we focused on the \acl{dp} of the proposed 2-\acs{wl}-\acs{gnn}.
In this supplementary section we describe how our approach can be implemented on \acfp{gpgpu}.

Efficient high-level software libraries for the implementation of vertex neighborhood convolution approaches such as \ac{gcn} or \ac{gin} already exist.
They describe convolutions via a message-passing abstraction in which vertex feature vectors are passed along their neighboring edges \citep[see][]{Battaglia2018}.
Since a message-passing model along edges is incompatible with the edge-pair neighborhoods of 2-\acs{wl}, a custom convolution implementation is required for 2-\acs{wl}-\acsp{gnn}.

For this purpose we propose a sparse 2-\acs{wl} graph representation which is inspired by the coordinate list adjacency format described by \citet{Fey2019}.
Given a neighborhood radius $r$, we encode a graph $G$ using the following two matrices:
\begin{enumerate}
	\item $Z_G^{(0)} \in \mathbb{R}^{m \times d^{(0)}}$:
		The initial feature matrix is represented directly as a dense floating point matrix with $m \coloneqq \left| \mathcal{E}_{G^r} \right|$ rows, each of which encodes the feature vector of an edge $e_{ij} \in \mathcal{E}_{G^r}$.
		Edge feature duplicates are prevented by only encoding edges with $i \leq j$ for some arbitrary vertex ordering of $G$.
	\item $R_G \in {[m]}^{\gamma \times 3}$:
		The reference matrix $R_G$ encodes the edge neighborhood information.
		It consists of $\gamma \coloneqq \sum_{e_{ij} \in \mathcal{E}_{G^r}} \left| \Gamma_{G^r}(v_i) \cap \Gamma_{G^r}(v_j) \right|$ rows, one for each 2-\acs{wl} neighbor $\ldblbrace e_{il}, e_{lj} \rdblbrace$ of each edge $e_{ij}$.
		Each neighbor row is a vector $(r_{\mathrm{L}}, r_{\Gamma,1}, r_{\Gamma,2}) \in {[m]}^3$ of three index pointers to rows in $Z_G^{(0)}$.
		$r_{\mathrm{L}}$ points to the row index of the feature vector of $e_{ij}$, while $r_{\Gamma,1}$ and $r_{\Gamma,2}$ point to the indices of $e_{il}$ and $e_{lj}$ respectively.
		We will refer to the three column vectors of $R_G$ as $R_{G, \mathrm{L}}$, $R_{G, \Gamma, 1}$ and $R_{G, \Gamma, 2}$.
\end{enumerate}
This encoding can also be used to represent graph batches by simply concatenating the rows of each graph's feature and reference matrices while shifting the index pointers to account for the concatenation offsets.
\Cref{fig:appendix:wl2-encoding} illustrates how such a batch encoding might look like.
\begin{figure}[ht]
	\centering
	\def\vertvec{$(1, 0)$}
	\def\edgevec{$(0, 1)$}
	\raisebox{-.5\height}{\includegraphics[height=11em]{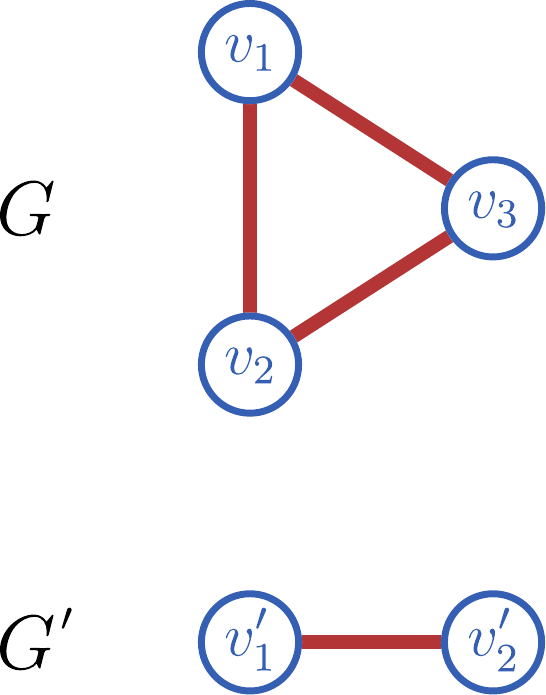}}\hspace{3em}%
	\begin{tabu}{@{\makebox[2.5em][c]{$\rownumber$\space}} c c}
		\multicolumn{1}{@{\makebox[2.5em][c]{idx.}} c}{edge} & $Z^{(0)}$ \\
		\toprule
		\rowfont{\color{t_blue}}$e_{11}$ & \vertvec\\ 
		\rowfont{\color{t_blue}}$e_{22}$ & \vertvec\\ 
		\rowfont{\color{t_blue}}$e_{33}$ & \vertvec\\ 
		\rowfont{\color{t_red}}$e_{12}$ & \edgevec\\ 
		\rowfont{\color{t_red}}$e_{13}$ & \edgevec\\ 
		\rowfont{\color{t_red}}$e_{23}$ & \edgevec\\ 
		\midrule
		\rowfont{\color{t_blue}}$e_{11}'$ & \vertvec\\ 
		\rowfont{\color{t_blue}}$e_{22}'$ & \vertvec\\ 
		\rowfont{\color{t_red}}$e_{12}'$ & \edgevec
	\end{tabu}\hspace{3em}{\def\arraystretch{0.81}\tiny\begin{tabu}{c c c}
		\rowfont{\footnotesize}$R_{\mathrm{L}}$ & $R_{\Gamma, 1}$ & $R_{\Gamma, 2}$ \\
		\toprule
		\rowfont{\color{t_blue}}$(1,$ & $1,$ & $1)$ \\ 
		\rowfont{\color{t_blue}}$(1,$ & $4,$ & $4)$ \\ 
		\rowfont{\color{t_blue}}$(1,$ & $5,$ & $5)$ \\ 
		\rowfont{\color{t_blue}}$(2,$ & $2,$ & $2)$ \\ 
		\rowfont{\color{t_blue}}$(2,$ & $4,$ & $4)$ \\ 
		\rowfont{\color{t_blue}}$(2,$ & $6,$ & $6)$ \\ 
		\rowfont{\color{t_blue}}$(3,$ & $3,$ & $3)$ \\ 
		\rowfont{\color{t_blue}}$(3,$ & $5,$ & $5)$ \\ 
		\rowfont{\color{t_blue}}$(3,$ & $6,$ & $6)$ \\ 
		\rowfont{\color{t_red}}$(4,$ & $1,$ & $4)$ \\ 
		\rowfont{\color{t_red}}$(4,$ & $4,$ & $2)$ \\ 
		\rowfont{\color{t_red}}$(4,$ & $5,$ & $6)$ \\ 
		\rowfont{\color{t_red}}$(5,$ & $1,$ & $5)$ \\ 
		\rowfont{\color{t_red}}$(5,$ & $5,$ & $3)$ \\ 
		\rowfont{\color{t_red}}$(5,$ & $4,$ & $6)$ \\ 
		\rowfont{\color{t_red}}$(6,$ & $2,$ & $6)$ \\ 
		\rowfont{\color{t_red}}$(6,$ & $6,$ & $3)$ \\ 
		\rowfont{\color{t_red}}$(6,$ & $4,$ & $5)$ \\ 
		\midrule
		\rowfont{\color{t_blue}}$(7,$ & $7,$ & $7)$ \\ 
		\rowfont{\color{t_blue}}$(7,$ & $9,$ & $9)$ \\ 
		\rowfont{\color{t_blue}}$(8,$ & $8,$ & $8)$ \\ 
		\rowfont{\color{t_blue}}$(8,$ & $9,$ & $9)$ \\ 
		\rowfont{\color{t_red}}$(9,$ & $7,$ & $9)$ \\ 
		\rowfont{\color{t_red}}$(9,$ & $9,$ & $8)$ 
	\end{tabu}%
	}\caption{
		Exemplary 2-\acs{wl} encoding of a batch of two small graphs.
	}\label{fig:appendix:wl2-encoding}
\end{figure}
After encoding a graph dataset as 2-\acs{wl} matrices, convolutions can be computed efficiently on \acp{gpgpu} via the common \textit{gather-scatter} pattern from parallel programming~\citep{He2007}.
The so-called $\mathit{gather}$ operator takes two inputs: A list $Z$ of $m$ row vectors and a list $R$ of $\gamma$ pointers into $Z$.
It returns a list $X$ of $\gamma$ row vectors $X[i] = Z[R[i]]$ for $i \in [\gamma]$.
The $\mathit{scatter}_{\Sigma}$ operator can be understood as the opposite of $\mathit{gather}$.
$\mathit{scatter}_{\Sigma}$ takes a list $X$ of $\gamma$ row vectors and a list $R$ of $\gamma$ pointers from the range $[m]$.
It returns a list $Z$ of $m$ row vectors $Z[i] = \sum_{j \in [\gamma] \land R[j] = i} X[j]$ for $i \in [m]$.

Using the $\mathit{gather}$ and $\mathit{scatter}_{\Sigma}$ operators, the 2-\acs{wl} convolution operator from \cref{defn:wl2conv:wl2-conv-step} can be computed via the following parallel algorithm:
\begin{algorithm}[H]
	\caption{Parallel Implementation of a 2-\acs{wl} Convolution Layer $S^{(t)}$}\label{algo:appendix:wl2-conv}
	\begin{algorithmic}[1]
		\Function{$S^{(t)}$}{$Z^{(t-1)} \in \mathbb{R}^{m \times d^{(t-1)}}, R \in {[m]}^{\gamma \times 3}$}
			\State{$Z_{\mathrm{L}} \coloneqq Z^{(t-1)} W_{\mathrm{L}}^{(t)}$}
			\Comment{Matrix multiply: $\mathbb{R}^{m \times d^{(t-1)}} \to \mathbb{R}^{m \times d^{(t)}}$}
			\State{$Z_{\mathrm{F}} \coloneqq Z^{(t-1)} W_{\mathrm{F}}^{(t)}$}
			\State{$Z_{\Gamma} \coloneqq Z^{(t-1)} W_{\Gamma}^{(t)}$}
			\State{$X_{\Gamma, 1} \coloneqq \mathit{gather}(Z_{\Gamma}, R_{\Gamma, 1})$}
			\Comment{Gather: $\mathbb{R}^{m \times d^{(t)}} \times {[m]}^{\gamma} \to \mathbb{R}^{\gamma \times d^{(t)}}$}
			\State{$X_{\Gamma, 2} \coloneqq \mathit{gather}(Z_{\Gamma}, R_{\Gamma, 2})$}
			\State{$X_{\Gamma} \coloneqq \sigma_{\Gamma}\left(X_{\Gamma, 1} + X_{\Gamma, 2}\right)$}
			\Comment{Element-wise operations}
			\State{$Z_{\Sigma\Gamma} \coloneqq \mathit{scatter}_{\Sigma}(X_{\Gamma}, R_{\mathrm{L}})$}
			\Comment{Scatter: $\mathbb{R}^{\gamma \times d^{(t)}} \times {[m]}^{\gamma} \to \mathbb{R}^{m \times d^{(t)}}$}
			\State{$Z^{(t)} \coloneqq \sigma\left(Z_{\mathrm{L}} + Z_{\mathrm{F}} \odot Z_{\Sigma\Gamma} \right)$}
			\Comment{Element-wise operations}
			\State{\Return{$Z^{(t)}$}}
		\EndFunction{}
	\end{algorithmic}
\end{algorithm}

\section{Evaluated Hyperparameter Grids}%
\label{sec:appendix:config-grid}

To tune the hyperparameters of the evaluated models, we used a regular grid search.
Depending on the type of model, different sets of hyperparameter configurations $\Theta$ were used.

\paragraph{Graph Kernels}
We used the \ac{svm} classifier from Scikit-learn to evaluate the graph kernel approaches.
We tuned only the regularization parameter \texttt{C} of this classifier;
the evaluated values are $\mathtt{C} \in \{ 1, \num{1e-1},\allowbreak \num{1e-2},\allowbreak \num{1e-3},\allowbreak \num{1e-4} \}$.
All other parameters were left at the default setting (using \texttt{scikit-learn 0.22.1}).

\paragraph{Baseline and \ac{gin}}
For the evaluation of the structure unaware baseline learner and \ac{gin}, we used the same hyperparameter configurations as \citet{Errica2020}.
We therefore refer to their work for a complete list of the tuned hyperparameters for those models.

\paragraph{2-\acs{gnn} and 2-\acs{wl}-\acs{gnn}}
We evaluated our implementations of 2-\acsp{gnn} as well as 2-\acs{wl}-\acsp{gnn} on the grid spanned by the following hyperparameter values:
\begin{itemize}[itemsep=2pt,parsep=2pt]
	\item \textbf{Number of convolutional layers $T \in \{ 3, 5 \}$:}
		This parameter describes only the depth of the stack of convolutional layers.
		The \ac{mlp} after the pooling layer is always configured with a single hidden layer.
	\item \textbf{Layer width $d \in \{ 32, 64 \}$:}
		This parameter describes the output dimensionalities $d = d^{(1)} = \cdots = d^{(T)}$ of the convolutional layers and (if applicable) also the hidden layer width of the final \ac{mlp} after the pooling layer.
	\item \textbf{Learning rate $\eta \in \{ \num{1e-2}, \num{1e-3}, \num{1e-4} \}$} of the Adam optimizer.
	\item \textbf{Activation functions $\sigma$ and $\sigma_{\Gamma}$} are set to the standard logistic function.
		However, for the evaluation of the synthetic TRIANGLE dataset we used ReLU instead because this choice led to improved and more consistent results in previous exploratory experiments.
	\item \textbf{Number of epochs $E$ and early stopping patience $p$} are set to $E = 1000$ and $p = 100$, except for the evaluation of the synthetic TRIANGLE dataset for which we used $E = 5000$ and $p = 1000$ to ensure model convergence.
\end{itemize}
Both, 2-\acsp{gnn} and 2-\acs{wl}-\acsp{gnn}, were evaluated using two different pooling layers which combine the edge feature vectors ${\{ z_{ij} \}}_{e_{ij} \in \mathcal{E}_{G^r}}$ into a graph feature representation $z_G$.
The $\mean$ pooling layer uses $z_G = \frac{1}{|\mathcal{E}_{G^r}|} \sum_{e_{ij}} z_{ij}$.
The $\wmean$ pooling layer extends this approach by incorporating attention scores $w_{ij} \in \mathbb{R}$ that are learned alongside $z_{ij}$ for each edge;
the graph feature representation is then defined as $z_G = \frac{1}{\sum_{e_{ij}} e^{w_{ij}}} \sum_{e_{ij}} e^{w_{ij}} z_{ij}$.

\section{Dataset Statistics and Descriptions}%
\label{sec:appendix:ds-stats}

\begin{table}[ht]
	\caption{Sizes of the evaluated binary classification datasets and their graphs.}\label{tbl:appendix:ds-stats}
	\centering\small
	\csvreader[
		tabular={lrrrrrrrrr},
		separator=comma,
		before reading=\setlength{\tabcolsep}{5pt},
		table head={%
			\multicolumn{1}{c}{} & \multicolumn{1}{c}{} & &\multicolumn{3}{c}{vertex count $\left|\mathcal{V}_G\right|$} & \multicolumn{3}{c}{edge count $\left|\mathcal{E}_G\right|$} & \multicolumn{1}{c}{vert.~deg.} \\%
			& \multirow{-2}{*}[-0.2em]{\shortstack[c]{no.\ of\\ graphs}} & \multirow{-2}{*}[-0.2em]{\shortstack[c]{vertex data\\{\scriptsize (feat.\ + lab.)}}} & $\min$ & $\mean$ & $\max$ & $\min$ & $\mean$ & $\max$ & $\mean \pm\, \sigma$ \\\toprule
		},
		table foot=\bottomrule,
		late after line=\\
	]{data/ds_stats.csv}%
	{name=\name,graph_count=\gcount,%
	node_count_min=\ncountmin,node_count_mean=\ncountmean,node_count_max=\ncountmax,%
	edge_count_min=\ecountmin,edge_count_mean=\ecountmean,edge_count_max=\ecountmax,%
	node_degree_min=\ndegmin,node_degree_mean=\ndegmean,node_degree_std=\ndegstd,node_degree_max=\ndegmax,%
	dim_node_features=\nfdim,dim_edge_features=\efdim, radius_mean=\radiusmean, radius_std=\radiusstd%
	}%
	{\textbf{\name}&%
	$\gcount$&%
	$\nfdim$&%
	$\ncountmin$&$\ncountmean$&$\ncountmax$&%
	$\ecountmin$&$\ecountmean$&$\ecountmax$&$\ndegmean \pm \ndegstd$%
	}
\end{table}

\paragraph{TRIANGLE}
The triangle detection dataset was generated by sampling three graphs with exactly one unicolored triangle uniformly at random for each possible combination of the following parameters:
The number of vertices (between 6 and 32), the vertex color proportions (either 50/50\%, 75/25\% or 25/75\% vertices with the colors \colorlabel{t_blue}{A}/\colorlabel{t_red}{B}), the graph density (${\left| \mathcal{V}_{G} \right|}^{-2} \left| \mathcal{E}_{G} \right| \in \{ \nicefrac{1}{4}, \nicefrac{1}{2} \}$) the graph class (add a triangle with either the color \colorlabel{t_blue}{A} or \colorlabel{t_red}{B}).

\paragraph{NCI1}
This dataset was made available by \citet{Shervashidze2011}.
It contains a balanced subset of molecule graphs that were originally published by the US \ac{nci}.
In each molecule graph, vertices correspond to atoms and edges to bonds between them.
The binary classes in this dataset describe whether a molecule is able to suppress or inhibit the growth of certain lung cancer and ovarian cancer cell lines in humans.

\paragraph{PROTEINS and D\&D}
The graphs in both the PROTEINS dataset~\citep{Borgwardt2005a} as well as the D\&D dataset~\citep{Dobson2003} represent proteins.
Each vertex corresponds to a so-called \ac{sse}, i.e.\ a certain molecular substructure.
An edge encodes either that two \acp{sse} are neighbors in the protein's amino-acid sequence or that those \acp{sse} are close to each other in 3D space.
Each protein graph is classified by whether it is an enzyme or not.
The main difference between the two datasets is their selection of vertex features/labels.

\paragraph{REDDIT}
This balanced dataset contains graphs that represent online discussion threads on the website Reddit~\citep{Yanardag2015}.
Each vertex corresponds to a user; an edge is drawn between two users iff.\ at least one of them replied to a comment of another.
Such social interaction graphs were sampled from two types of subreddits:
Question/answer-based and discussion-based.
The classification goal is to predict from which type of subreddit a given graph was sampled.

\paragraph{IMDB}
This dataset contains so-called \textit{ego-networks} of movie actors~\citep{Yanardag2015}.
Vertices in such networks represent actors and edges encode whether two actors starred in the same movie.
The graphs in the dataset are derived from the actors starring in either action or romance movies.
The classification goal for each graph is to predict the movie genre it was derived from.

\section{Influence of the Neighborhood Radius on the Predictive Performance}%
\label{sec:appendix:radius}

As described in \cref{sec:wl2conv}, the neighborhood radius $r \in \mathbb{N}$ determines the number of convolved edges feature vectors, i.e.\ the number of rows in $Z^{(t)} \in \mathbb{R}^{\left|\mathcal{E}_{G^r}\right| \times d^{(t)}}$ (see \cref{fig:appendix:neighborhood-radius}).
We will now analyze the influence of $r$ on the accuracy of 2-\acs{wl}-\acsp{gnn}.
The theory suggests that the \ac{dp} of the 2-\acs{wl} convolution is increased by increasing the radius; e.g., the proof of 2-\acs{wl}'s cycle counting ability \citep{Fuerer2017} depends heavily on the structural information carried by the indirect edge features/colors that are only present if $r > 1$.
\begin{figure}[ht]
	\centering
	\includegraphics[width=0.9\linewidth]{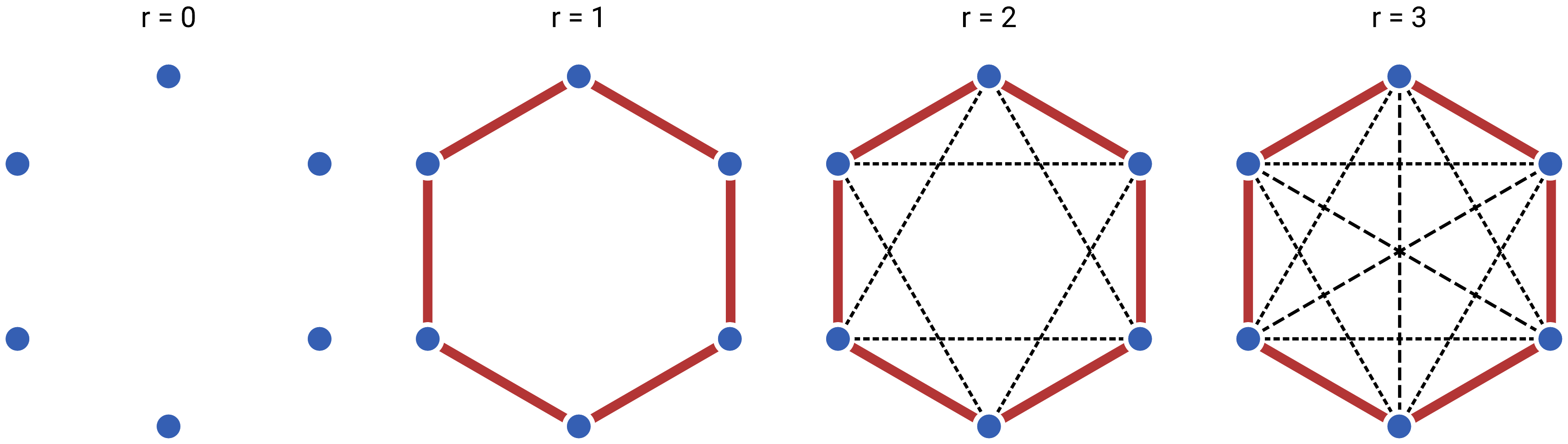}
	\caption[Illustration of the powers of the six-cycle graph.]{
		Illustration of the powers of the six-cycle graph for varying $r$.
		The self-loop edge at each of the vertices is not explicitly shown.
		For $r = 3$ all possible edges between the six vertices will be considered, just as in the original 2-\acs{wl} algorithm.
	}\label{fig:appendix:neighborhood-radius}
\end{figure}
\newcommand{\wlRadiusPlot}[4][width=0.22\linewidth]{%
\begin{tikzpicture}
	\begin{axis}[
		height=0.3\linewidth,
		#1,
		label style={font=\tiny},
		tick label style={font=\tiny},
		ymajorgrids,
		ytick style={draw=none},
		ylabel shift=-6pt,
		xlabel shift=-5pt,
		axis line style={gray},
		mark size={1.5pt},
		title style={yshift=-0.3em},
		#3
	]
		\addplot [color=t_blue, only marks, mark=*]
		plot[error bars/.cd, y dir=both, y explicit]
		table [x=r, y=#2TrainMean, y error=#2TrainStd, col sep=semicolon] {data/wl2_radii_#4.csv};
		\label{pgfplots:appendix:wl-radius:#4-#2-train}
		\addplot [color=t_red, only marks, mark=square*]
		plot[error bars/.cd, y dir=both, y explicit]
		table [x=r, y=#2TestMean, y error=#2TestStd, col sep=semicolon] {data/wl2_radii_#4.csv};
		\label{pgfplots:appendix:wl-radius:#4-#2-test}
	\end{axis}
\end{tikzpicture}} 
\begin{figure}[ht]
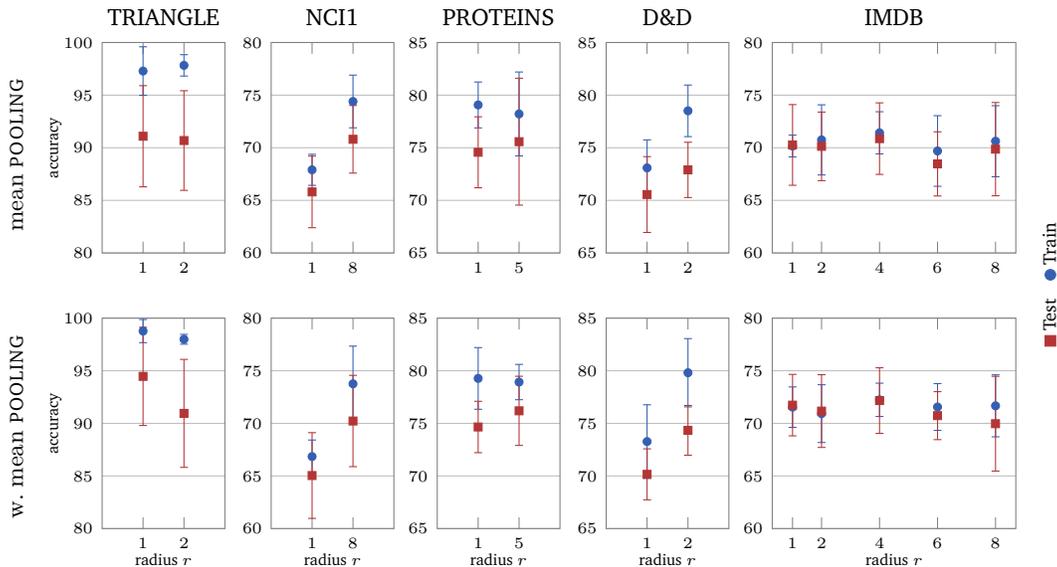

	\centering
	\setlength{\tabcolsep}{1pt}
	\begin{tabular}{m{1em}cccccc}
		\raisebox{0.33\height}[0pt][0pt]{\rotatebox{90}{\small\textsc{$\mean$~pooling}}} &
		\wlRadiusPlot{triangle}{title={\small TRIANGLE},ylabel={accuracy},xtick={1,2},xmin=0,xmax=3,ymin=80,ymax=100,try min ticks=4}{mean} &
		\wlRadiusPlot{nci}{title={\small NCI1},xtick={1,8},xmin=-6,xmax=15,ymin=60,ymax=80,try min ticks=4}{mean} &
		\wlRadiusPlot{proteins}{title={\small PROTEINS},xtick={1,5},xmin=-3,xmax=9,ymin=65,ymax=85,try min ticks=4}{mean} &
		\wlRadiusPlot{dd}{title={\small D\&D},xtick={1,2},xmin=0,xmax=3,ymin=65,ymax=85,try min ticks=4}{mean} &
		\wlRadiusPlot[width=0.33\linewidth]{imdb}{title={\small IMDB},xtick={1,2,4,6,8},ymin=60,ymax=80,try min ticks=4}{mean} &
		\multirow{2}{*}[1.5em]{\hspace{8pt}\rotatebox[origin=c]{90}{\scriptsize%
			\ref{pgfplots:appendix:wl-radius:mean-triangle-test}~Test\quad 
			\ref{pgfplots:appendix:wl-radius:mean-triangle-train}~Train 
		}} \\
		\raisebox{0.3\height}[0pt][0pt]{\rotatebox{90}{\small\textsc{$\mathrm{w.\ mean}$~pooling}}} &
		\wlRadiusPlot{triangle}{title={\hphantom{\small TRIANGLE}},ylabel={accuracy},xtick={1,2},xmin=0,xmax=3,ymin=80,ymax=100,try min ticks=4, xlabel={radius $r$}}{sam} &
		\wlRadiusPlot{nci}{xtick={1,8},xmin=-6,xmax=15,ymin=60,ymax=80,try min ticks=4, xlabel={radius $r$}}{sam} &
		\wlRadiusPlot{proteins}{title={\hphantom{\small PROTEINS}},xtick={1,5},xmin=-3,xmax=9,ymin=65,ymax=85,try min ticks=4, xlabel={radius $r$}}{sam} &
		\wlRadiusPlot{dd}{xtick={1,2},xmin=0,xmax=3,ymin=65,ymax=85,try min ticks=4, xlabel={radius $r$}}{sam} &
		\wlRadiusPlot[width=0.33\linewidth]{imdb}{xtick={1,2,4,6,8},ymin=60,ymax=80,try min ticks=4, xlabel={radius $r$}}{sam} &
	\end{tabular}
	\caption[Accuracy of the 2-\acs{wl}-\acs{gnn} with varying neighborhood radii.]{
		Accuracy of the 2-\acs{wl}-\acs{gnn} with varying neighborhood radii $r$.
		All datasets were evaluated on $r = 1$ and the highest radius for which the 2-\acs{wl} graph encodings would still fit into memory;
		the REDDIT dataset only fit into memory for $r = 1$, therefore it is not shown here.
	}\label{fig:appendix:wl2-gnn-radius}
\end{figure}

\Cref{fig:appendix:wl2-gnn-radius} shows\footnote{
    Note that the accuracies in this figure are lower than those in the evaluation section of the main paper because a different 2-\acs{wl}-\acs{gnn} architecture is used here; namely, no \ac{mlp} is applied after pooling.
} that, in practice, a neighborhood radius $r > 1$ does correlate with a higher training and test accuracy on the NCI1 and D\&D datasets; howerver, on the IMDB and PROTEINS datasets this is not the case.
This difference is interesting because NCI1 (molecular structures) and D\&D (protein sequences) contain more cyclic graphs, while IMDB (ego-network structures) and PROTEINS (protein sequences) consist of more tree- or list-like graphs (see \cref{sec:appendix:ds-stats}).
Even though the PROTEINS and D\&D datasets both contain protein sequences, we find that the protein sequences in the PROTEINS dataset are very ``list-like'' with much fewer large cycles than in the proteins structures of the D\&D dataset (compare the vertex and edge count statistics of both datasets in \cref{tbl:appendix:ds-stats}).
\Cref{fig:appendix:proteins-dd-diff} illustrates this difference.
This leads us the the hypothesis that 2-\acs{wl}-\acsp{gnn} with a neighborhood radius of $r > 1$ are able to improve their real-world performance over that achieved with $r = 1$ by detecting cyclic constituents in graphs.
Due to the limited number of evaluation results, further investigations are required to verify this hypothesis.
\begin{figure}[h]
	\centering
	\includegraphics[width=0.8\linewidth]{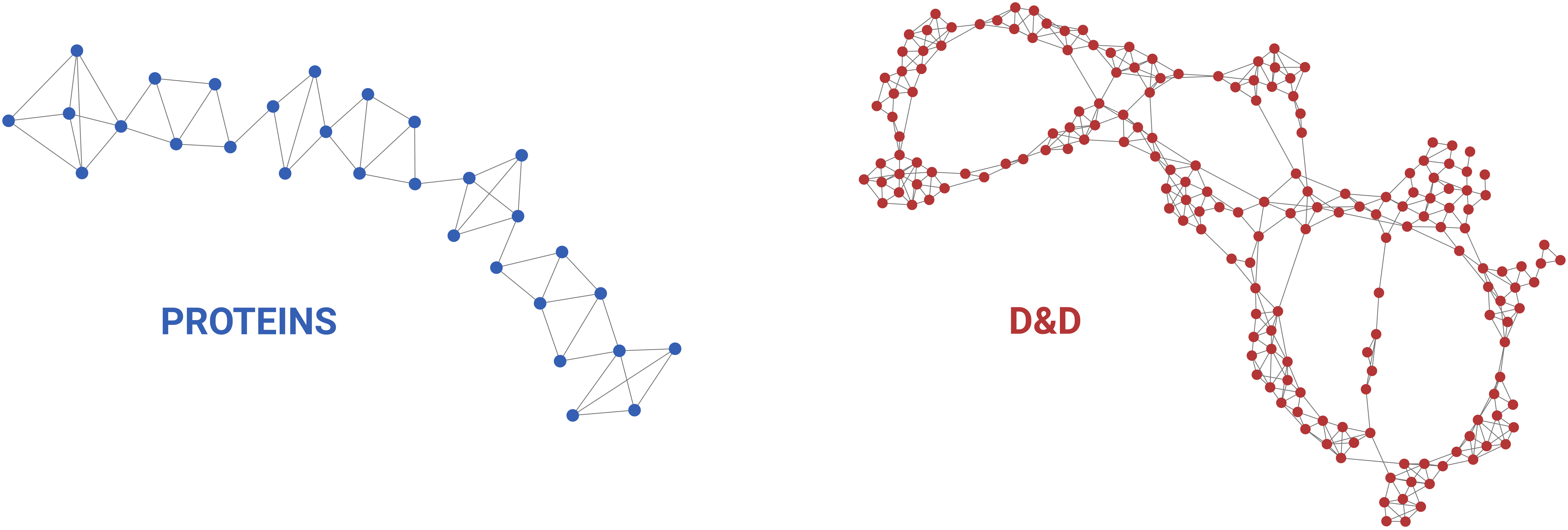}
	\caption{
		Two samples illustrating the difference between the PROTEINS and D\&D datasets.
	}\label{fig:appendix:proteins-dd-diff}
\end{figure}

\section{Empirical Runtime Evaluation}%
\label{sec:appendix:epoch-times}

\begin{figure}[ht]
	\centering
    \raisebox{-.5\height}{\footnotesize\begin{tikzpicture}
        \begin{axis}[
            width=0.40\linewidth,
            height=0.32\linewidth,
            xmode=log,
            ymode=log,
            log basis x=2,
            ylabel={epoch duration (\si{\milli\s})},
            xlabel={vertex count $n$},
            legend style={at={(0.04, 0.95)}, anchor=north west},
            ymajorgrids,
            try min ticks=4,
            ymin=10,
            ymax=1700
        ]
            \addplot [no markers, color=t_blue] table [x=N, y=r3, col sep=comma] {data/epoch_times_N.csv};
            \addplot [no markers, color=t_green] table [x=N, y=r2, col sep=comma] {data/epoch_times_N.csv};
            \addplot [no markers, color=t_red] table [x=N, y=r1, col sep=comma] {data/epoch_times_N.csv};
            \addplot [no markers, color=black] table [x=N, y=gin, col sep=comma] {data/epoch_times_N.csv};
            \legend{$r=3$, $r=2$, $r=1$, \acs*{gin}}
        \end{axis}
    \end{tikzpicture}\hspace{3pt}
    \begin{tikzpicture}
        \begin{axis}[
            width=0.40\linewidth,
            height=0.32\linewidth,
            xlabel={vertex degree $d$},
            ymajorgrids,
            try min ticks=5
        ]
            \addplot [no markers, color=t_blue] table [x=d, y=r3, col sep=comma] {data/epoch_times_d.csv};
            \addplot [no markers, color=t_green] table [x=d, y=r2, col sep=comma] {data/epoch_times_d.csv};
            \addplot [no markers, color=t_red] table [x=d, y=r1, col sep=comma] {data/epoch_times_d.csv};
            \addplot [no markers, color=black] table [x=d, y=gin, col sep=comma] {data/epoch_times_d.csv};
        \end{axis}
    \end{tikzpicture}}\hfill
    \raisebox{1em}{\footnotesize\begin{filecontents*}{data/epoch_times_scaling.csv}
d,r1,r2,r3
2,2.0,3.2,4.5
3,2.4,5.2,8.9
4,2.8,8.7,20.7
5,3.1,14.8,35.6
6,3.5,20.4,58.2
\end{filecontents*}
\csvreader[
		tabular={rrrr},
		separator=comma,
		before reading=\setlength{\tabcolsep}{2pt},
		table head={%
            $d$ & $r=1$ & $r=2$ & $r=3$ \\\toprule
		},
		table foot=\bottomrule,
		late after line=\\
	]{data/epoch_times_scaling.csv}%
	{d=\d,r1=\ra,r2=\rb,r3=\rc%
	}%
	{$\mathbf{\d}$ & $\ra$ & $\rb$ & $\rc$}}
    \caption{
		Comparison of the mean training epoch durations of 2-\acs{wl}-\acs{gnn} and \ac{gin} for varying graph sizes $n$, vertex degrees $d$ and neighborhood radii $r$.
		\textit{(Left)} Durations for varying vertex counts $n \in \{ 2^4, \dots, 2^{14} \}$ with a fixed degree $d = 2$.
		\textit{(Middle)} Durations for varying vertex degrees $d$ with a fixed size $n = 1024$.
		\textit{(Right)} Table of the factors by which 2-\acs{wl}-\acs{gnn} is slower than \ac{gin} for $n = 1024$.
	}\label{fig:appendix:epoch-times}
\end{figure}
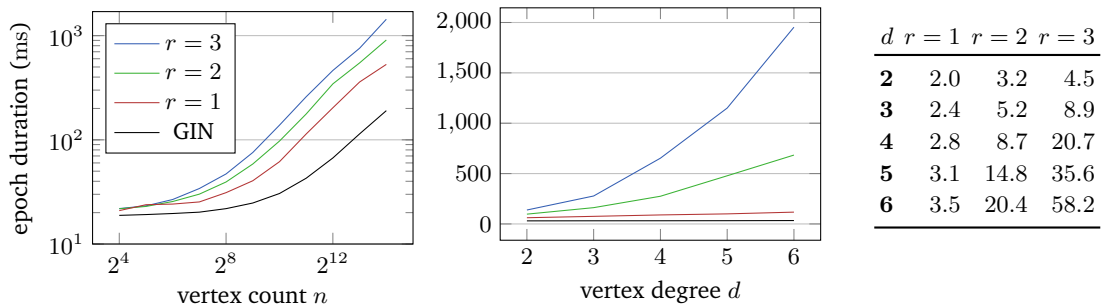
In \cref{sec:appendix:radius} we evaluated the relation between the neighborhood radius $r$ and the resulting accuracy.
As described in \cref{sec:wl2conv}, the neighborhood radius also affects the 2-\acs{wl} convolution runtime which is bounded by $\mathcal{O}(n d^{2r})$, with $n$ denoting the vertex count and $d$ being the maximum vertex degree.

We will now verify this bound experimentally.
\Cref{fig:appendix:epoch-times} shows the duration of a single training epoch of a 2-\acs{wl}-\acs{gnn} and a \ac{gin} model with the same depth and a similar number of learnable parameters ($\approx 2200$).
The time for each combination of $n$, $d$ and $r$ was obtained by taking the mean duration of the first 100 training epochs using a dataset of 100 synthesized regular graphs with size $n$ and uniform vertex degree $d$.
The experiment was implemented in TensorFlow and executed on a single Nvidia 1080 Ti GPU.
Since the coefficient of variation for the 100 samples of each $(n, d, r)$ combination is $\leq 5\%$, no error bars are shown for visual clarity.

In the left plot of \cref{fig:appendix:epoch-times} we see that the epoch durations are dominated by constant costs for $n < 2^{10}$; for $n \geq 2^{10}$ the expected linearity in $n$ can be observed.
For 2-\acs{wl}-\acsp{gnn} we expect that the slope of the epoch durations is described by $\mathcal{O}(d^{2r})$.
The roughly uniform y-offsets of the $r = 1$, $r = 2$ and $r = 3$ curves in the left log-log plot are in line with this expectation.
Additionally, the middle plot confirms the expected power law relation between the epoch duration and the vertex degree $d$.

The table on the right shows how much slower the 2-\acs{wl}-\acs{gnn} is compared to the 1-\acs{wl} bounded \ac{gin} architecture.
Note however that we considered the worst-case scenario in which all vertices reach the upper degree bound $d$.
Many real-world datasets consist of graphs with only a few high-degree vertices.
Looking at the NCI1, PROTEINS, D\&D, REDDIT and IMDB datasets, which have mean vertex degrees that range between $2.2$ and $9.8$ (see \cref{tbl:appendix:ds-stats}), the table in \cref{fig:appendix:epoch-times} would suggest that a 2-\acs{wl}-\acs{gnn} with $r = 1$ is at least $2$ or $3$ times slower than a \ac{gin} on those datasets.
However, in reality the slowdown is only $\frac{228\si{\milli\s}}{198\si{\milli\s}} \approx 1.15$ for NCI1, $\frac{262\si{\milli\s}}{206\si{\milli\s}} \approx 1.27$ for PROTEINS, $\frac{765\si{\milli\s}}{514\si{\milli\s}} \approx 1.49$ for D\&D, $\frac{1081\si{\milli\s}}{804\si{\milli\s}} \approx 1.35$ for REDDIT and $\frac{96\si{\milli\s}}{61\si{\milli\s}} \approx 1.57$ for IMDB.
For neighborhood radii $r > 1$, the real-world slowdown factor can be significantly smaller than the worst-case slowdown as well: $\frac{964\si{\milli\s}}{198\si{\milli\s}} \approx 4.87$ for NCI1 with $r = 8$, $\frac{670\si{\milli\s}}{206\si{\milli\s}} \approx 3.25$ for PROTEINS with $r = 5$, $\frac{1386\si{\milli\s}}{514\si{\milli\s}} \approx 2.70$ for D\&D with $r = 2$ and $\frac{239\si{\milli\s}}{61\si{\milli\s}} \approx 3.92$ for IMDB with $r = 4$.
Consequently, 2-\acs{wl}-\acsp{gnn} appear to be computationally feasible in practice despite the large worst-case slowdown factors in \cref{fig:appendix:epoch-times}.
Additionally, as we saw in \cref{sec:appendix:radius}, a neighborhood radius $r > 1$ is not always necessary to reach optimal predictive performance;
the computationally cheap choice of $r = 1$ should therefore always be considered.

\section{Fold-wise Accuracy Deltas}%
\label{sec:appendix:fold-diffs}

Due to the relatively small sizes of the evaluated benchmark datasets, the variance of the test accuracies across different folds is quite large.
When directly comparing the mean accuracies of two learners, it is often impossible to tell whether one consistently outperforms the other.
We therefore now list the mean and standard deviations of the fold-wise test accuracy differences of all pairs of learners for all datasets.
This effectively removes the variance introduced by ``easy'' and ``hard'' folds on which all learners might tend to perform consistently better/worse.

In the following \crefrange{tbl:appendix:diff-triangle}{tbl:appendix:diff-imdb} we show accuracy differences as \textit{row accuracy} minus \textit{column accuracy}.
For each row $i$ and column $j$, the corresponding cell $(i,j)$ is highlighted in \textcolor{t_red}{red} or \textcolor{t_darkgreen}{green} iff.\ the learner~$i$ performs consistently \textcolor{t_red}{worse} (or \textcolor{t_darkgreen}{better} respectively) than $j$ with a significance level of $2\sigma$.
To compute the deltas for 2-\acs{wl}-\acsp{gnn}, the same neighborhood radii as in the evaluation section of the main paper are used, i.e.\ $r = 2$ for the synthetic triangle detection dataset and $r = 8$, $5$, $2$, $1$ and $4$ for NCI1, PROTEINS, D\&D, REDDIT and IMDB respectively.

{\def\wmean{\mathrm{w.m.}}%
\captionsetup[table]{name=Matrix}%
\begin{table}[ht]
	\caption{Fold-wise accuracy delta means and standard deviations on the triangle dataset.}\label[mat]{tbl:appendix:diff-triangle}
	\centering\small
{\setlength\tabcolsep{2.5pt}\setlength{\extrarowheight}{2pt}%
\begin{tabular}{lcccccccccc}
& \rotatebox[origin=l]{90}{\textbf{WL\textsubscript{ST}} ($T=3$)} & \rotatebox[origin=l]{90}{\textbf{WL\textsubscript{SP}} ($T=3$)} & \rotatebox[origin=l]{90}{\textbf{2-LWL} ($T=3$)} & \rotatebox[origin=l]{90}{\textbf{2-GWL} ($T=3$)} & \rotatebox[origin=l]{90}{\textbf{Baseline} ($\mathrm{sum}$)} & \rotatebox[origin=l]{90}{\textbf{GIN} ($\mathrm{sum}$)} & \rotatebox[origin=l]{90}{\textbf{2-GNN} ($\mean$)} & \rotatebox[origin=l]{90}{\textbf{2-GNN} ($\wmean$)} & \rotatebox[origin=l]{90}{\textbf{2-WL-GNN} ($\mean$)} & \rotatebox[origin=l]{90}{\textbf{2-WL-GNN} ($\wmean$)} \\
\textbf{WL\textsubscript{ST}} ($T=3$)& & {\tiny\Vectorstack{-11\\ \pm 13}} & {\tiny\Vectorstack{+0.4\\ \pm 9.6}} & {\tiny\Vectorstack{-4.9\\ \pm 16}} & {\tiny\Vectorstack{+12\\ \pm 12}} & {\tiny\Vectorstack{-13\\ \pm 10}} & {\tiny\Vectorstack{-20\\ \pm 14}} & {\tiny\Vectorstack{-25\\ \pm 14}} & \cellcolor{t_red!25}\textcolor{t_darkred}{{\tiny\Vectorstack{-36\\ \pm 17}}} & \cellcolor{t_red!25}\textcolor{t_darkred}{{\tiny\Vectorstack{-42\\ \pm 11}}} \\[2pt]
\textbf{WL\textsubscript{SP}} ($T=3$)&{\tiny\Vectorstack{+11\\ \pm 13}} &  & {\tiny\Vectorstack{+11\\ \pm 13}} & {\tiny\Vectorstack{+6.1\\ \pm 9.2}} & {\tiny\Vectorstack{+23\\ \pm 16}} & {\tiny\Vectorstack{-2.0\\ \pm 14}} & {\tiny\Vectorstack{-8.8\\ \pm 13}} & {\tiny\Vectorstack{-14\\ \pm 14}} & \cellcolor{t_red!25}\textcolor{t_darkred}{{\tiny\Vectorstack{-25\\ \pm 12}}} & \cellcolor{t_red!25}\textcolor{t_darkred}{{\tiny\Vectorstack{-31\\ \pm 11}}} \\[2pt]
\textbf{2-LWL} ($T=3$)&{\tiny\Vectorstack{-0.4\\ \pm 9.6}} & {\tiny\Vectorstack{-11\\ \pm 13}} &  & {\tiny\Vectorstack{-5.3\\ \pm 11}} & {\tiny\Vectorstack{+12\\ \pm 7.5}} & {\tiny\Vectorstack{-13\\ \pm 6.7}} & {\tiny\Vectorstack{-20\\ \pm 13}} & \cellcolor{t_red!25}\textcolor{t_darkred}{{\tiny\Vectorstack{-25\\ \pm 9.4}}} & \cellcolor{t_red!25}\textcolor{t_darkred}{{\tiny\Vectorstack{-36\\ \pm 10}}} & \cellcolor{t_red!25}\textcolor{t_darkred}{{\tiny\Vectorstack{-43\\ \pm 6.8}}} \\[2pt]
\textbf{2-GWL} ($T=3$)&{\tiny\Vectorstack{+4.9\\ \pm 16}} & {\tiny\Vectorstack{-6.1\\ \pm 9.2}} & {\tiny\Vectorstack{+5.3\\ \pm 11}} &  & {\tiny\Vectorstack{+17\\ \pm 15}} & {\tiny\Vectorstack{-8.2\\ \pm 13}} & {\tiny\Vectorstack{-15\\ \pm 14}} & {\tiny\Vectorstack{-20\\ \pm 13}} & \cellcolor{t_red!25}\textcolor{t_darkred}{{\tiny\Vectorstack{-31\\ \pm 9.0}}} & \cellcolor{t_red!25}\textcolor{t_darkred}{{\tiny\Vectorstack{-38\\ \pm 9.0}}} \\[2pt]
\textbf{Baseline} ($\mathrm{sum}$)&{\tiny\Vectorstack{-12\\ \pm 12}} & {\tiny\Vectorstack{-23\\ \pm 16}} & {\tiny\Vectorstack{-12\\ \pm 7.5}} & {\tiny\Vectorstack{-17\\ \pm 15}} &  & \cellcolor{t_red!25}\textcolor{t_darkred}{{\tiny\Vectorstack{-25\\ \pm 8.4}}} & \cellcolor{t_red!25}\textcolor{t_darkred}{{\tiny\Vectorstack{-32\\ \pm 14}}} & \cellcolor{t_red!25}\textcolor{t_darkred}{{\tiny\Vectorstack{-37\\ \pm 9.2}}} & \cellcolor{t_red!25}\textcolor{t_darkred}{{\tiny\Vectorstack{-48\\ \pm 14}}} & \cellcolor{t_red!25}\textcolor{t_darkred}{{\tiny\Vectorstack{-55\\ \pm 8.5}}} \\[2pt]
\textbf{GIN} ($\mathrm{sum}$)&{\tiny\Vectorstack{+13\\ \pm 10}} & {\tiny\Vectorstack{+2.0\\ \pm 14}} & {\tiny\Vectorstack{+13\\ \pm 6.7}} & {\tiny\Vectorstack{+8.2\\ \pm 13}} & \cellcolor{t_green!25}\textcolor{t_darkgreen}{{\tiny\Vectorstack{+25\\ \pm 8.4}}} &  & {\tiny\Vectorstack{-6.8\\ \pm 13}} & {\tiny\Vectorstack{-12\\ \pm 9.5}} & {\tiny\Vectorstack{-23\\ \pm 14}} & \cellcolor{t_red!25}\textcolor{t_darkred}{{\tiny\Vectorstack{-29\\ \pm 7.3}}} \\[2pt]
\textbf{2-GNN} ($\mean$)&{\tiny\Vectorstack{+20\\ \pm 14}} & {\tiny\Vectorstack{+8.8\\ \pm 13}} & {\tiny\Vectorstack{+20\\ \pm 13}} & {\tiny\Vectorstack{+15\\ \pm 14}} & \cellcolor{t_green!25}\textcolor{t_darkgreen}{{\tiny\Vectorstack{+32\\ \pm 14}}} & {\tiny\Vectorstack{+6.8\\ \pm 13}} &  & {\tiny\Vectorstack{-5.0\\ \pm 6.6}} & {\tiny\Vectorstack{-16\\ \pm 11}} & \cellcolor{t_red!25}\textcolor{t_darkred}{{\tiny\Vectorstack{-23\\ \pm 9.8}}} \\[2pt]
\textbf{2-GNN} ($\wmean$)&{\tiny\Vectorstack{+25\\ \pm 14}} & {\tiny\Vectorstack{+14\\ \pm 14}} & \cellcolor{t_green!25}\textcolor{t_darkgreen}{{\tiny\Vectorstack{+25\\ \pm 9.4}}} & {\tiny\Vectorstack{+20\\ \pm 13}} & \cellcolor{t_green!25}\textcolor{t_darkgreen}{{\tiny\Vectorstack{+37\\ \pm 9.2}}} & {\tiny\Vectorstack{+12\\ \pm 9.5}} & {\tiny\Vectorstack{+5.0\\ \pm 6.6}} &  & {\tiny\Vectorstack{-11\\ \pm 8.8}} & \cellcolor{t_red!25}\textcolor{t_darkred}{{\tiny\Vectorstack{-18\\ \pm 7.2}}} \\[2pt]
\textbf{2-WL-GNN} ($\mean$)&\cellcolor{t_green!25}\textcolor{t_darkgreen}{{\tiny\Vectorstack{+36\\ \pm 17}}} & \cellcolor{t_green!25}\textcolor{t_darkgreen}{{\tiny\Vectorstack{+25\\ \pm 12}}} & \cellcolor{t_green!25}\textcolor{t_darkgreen}{{\tiny\Vectorstack{+36\\ \pm 10}}} & \cellcolor{t_green!25}\textcolor{t_darkgreen}{{\tiny\Vectorstack{+31\\ \pm 9.0}}} & \cellcolor{t_green!25}\textcolor{t_darkgreen}{{\tiny\Vectorstack{+48\\ \pm 14}}} & {\tiny\Vectorstack{+23\\ \pm 14}} & {\tiny\Vectorstack{+16\\ \pm 11}} & {\tiny\Vectorstack{+11\\ \pm 8.8}} &  & {\tiny\Vectorstack{-6.5\\ \pm 8.7}} \\[2pt]
\textbf{2-WL-GNN} ($\wmean$)&\cellcolor{t_green!25}\textcolor{t_darkgreen}{{\tiny\Vectorstack{+42\\ \pm 11}}} & \cellcolor{t_green!25}\textcolor{t_darkgreen}{{\tiny\Vectorstack{+31\\ \pm 11}}} & \cellcolor{t_green!25}\textcolor{t_darkgreen}{{\tiny\Vectorstack{+43\\ \pm 6.8}}} & \cellcolor{t_green!25}\textcolor{t_darkgreen}{{\tiny\Vectorstack{+38\\ \pm 9.0}}} & \cellcolor{t_green!25}\textcolor{t_darkgreen}{{\tiny\Vectorstack{+55\\ \pm 8.5}}} & \cellcolor{t_green!25}\textcolor{t_darkgreen}{{\tiny\Vectorstack{+29\\ \pm 7.3}}} & \cellcolor{t_green!25}\textcolor{t_darkgreen}{{\tiny\Vectorstack{+23\\ \pm 9.8}}} & \cellcolor{t_green!25}\textcolor{t_darkgreen}{{\tiny\Vectorstack{+18\\ \pm 7.2}}} & {\tiny\Vectorstack{+6.5\\ \pm 8.7}} & 
\end{tabular}}

\end{table}
\begin{table}[ht]
	\caption{Fold-wise accuracy delta means and standard deviations on NCI1.}\label[mat]{tbl:appendix:diff-nci}
	\centering\small
{\setlength\tabcolsep{2.5pt}\setlength{\extrarowheight}{2pt}%
\begin{tabular}{lcccccccccc}
& \rotatebox[origin=l]{90}{\textbf{WL\textsubscript{ST}} ($T=1$)} & \rotatebox[origin=l]{90}{\textbf{WL\textsubscript{ST}} ($T=3$)} & \rotatebox[origin=l]{90}{\textbf{2-LWL} ($T=3$)} & \rotatebox[origin=l]{90}{\textbf{2-GWL} ($T=3$)} & \rotatebox[origin=l]{90}{\textbf{Baseline} ($\mathrm{sum}$)} & \rotatebox[origin=l]{90}{\textbf{GIN} ($\mathrm{sum}$)} & \rotatebox[origin=l]{90}{\textbf{2-GNN} ($\mean$)} & \rotatebox[origin=l]{90}{\textbf{2-GNN} ($\wmean$)} & \rotatebox[origin=l]{90}{\textbf{2-WL-GNN} ($\mean$)} & \rotatebox[origin=l]{90}{\textbf{2-WL-GNN} ($\wmean$)} \\
\textbf{WL\textsubscript{ST}} ($T=1$)& & \cellcolor{t_red!25}\textcolor{t_darkred}{{\tiny\Vectorstack{-11\\ \pm 1.3}}} & {\tiny\Vectorstack{-2.8\\ \pm 1.5}} & {\tiny\Vectorstack{+2.3\\ \pm 2.4}} & \cellcolor{t_green!25}\textcolor{t_darkgreen}{{\tiny\Vectorstack{+6.2\\ \pm 2.9}}} & {\tiny\Vectorstack{-3.5\\ \pm 2.2}} & {\tiny\Vectorstack{-1.9\\ \pm 2.9}} & {\tiny\Vectorstack{-4.4\\ \pm 2.3}} & {\tiny\Vectorstack{+1.6\\ \pm 2.6}} & {\tiny\Vectorstack{+0.4\\ \pm 2.1}} \\[2pt]
\textbf{WL\textsubscript{ST}} ($T=3$)&\cellcolor{t_green!25}\textcolor{t_darkgreen}{{\tiny\Vectorstack{+11\\ \pm 1.3}}} &  & \cellcolor{t_green!25}\textcolor{t_darkgreen}{{\tiny\Vectorstack{+8.1\\ \pm 1.6}}} & \cellcolor{t_green!25}\textcolor{t_darkgreen}{{\tiny\Vectorstack{+13\\ \pm 1.8}}} & \cellcolor{t_green!25}\textcolor{t_darkgreen}{{\tiny\Vectorstack{+17\\ \pm 2.9}}} & \cellcolor{t_green!25}\textcolor{t_darkgreen}{{\tiny\Vectorstack{+7.3\\ \pm 2.7}}} & \cellcolor{t_green!25}\textcolor{t_darkgreen}{{\tiny\Vectorstack{+8.9\\ \pm 2.3}}} & \cellcolor{t_green!25}\textcolor{t_darkgreen}{{\tiny\Vectorstack{+6.5\\ \pm 2.0}}} & \cellcolor{t_green!25}\textcolor{t_darkgreen}{{\tiny\Vectorstack{+12\\ \pm 2.5}}} & \cellcolor{t_green!25}\textcolor{t_darkgreen}{{\tiny\Vectorstack{+11\\ \pm 2.0}}} \\[2pt]
\textbf{2-LWL} ($T=3$)&{\tiny\Vectorstack{+2.8\\ \pm 1.5}} & \cellcolor{t_red!25}\textcolor{t_darkred}{{\tiny\Vectorstack{-8.1\\ \pm 1.6}}} &  & \cellcolor{t_green!25}\textcolor{t_darkgreen}{{\tiny\Vectorstack{+5.1\\ \pm 2.3}}} & \cellcolor{t_green!25}\textcolor{t_darkgreen}{{\tiny\Vectorstack{+9.0\\ \pm 3.0}}} & {\tiny\Vectorstack{-0.7\\ \pm 2.7}} & {\tiny\Vectorstack{+0.9\\ \pm 2.0}} & {\tiny\Vectorstack{-1.6\\ \pm 2.3}} & {\tiny\Vectorstack{+4.4\\ \pm 2.5}} & {\tiny\Vectorstack{+3.2\\ \pm 2.5}} \\[2pt]
\textbf{2-GWL} ($T=3$)&{\tiny\Vectorstack{-2.3\\ \pm 2.4}} & \cellcolor{t_red!25}\textcolor{t_darkred}{{\tiny\Vectorstack{-13\\ \pm 1.8}}} & \cellcolor{t_red!25}\textcolor{t_darkred}{{\tiny\Vectorstack{-5.1\\ \pm 2.3}}} &  & {\tiny\Vectorstack{+3.9\\ \pm 4.1}} & {\tiny\Vectorstack{-5.8\\ \pm 3.6}} & {\tiny\Vectorstack{-4.3\\ \pm 2.5}} & \cellcolor{t_red!25}\textcolor{t_darkred}{{\tiny\Vectorstack{-6.7\\ \pm 2.5}}} & {\tiny\Vectorstack{-0.8\\ \pm 2.7}} & {\tiny\Vectorstack{-1.9\\ \pm 3.1}} \\[2pt]
\textbf{Baseline} ($\mathrm{sum}$)&\cellcolor{t_red!25}\textcolor{t_darkred}{{\tiny\Vectorstack{-6.2\\ \pm 2.9}}} & \cellcolor{t_red!25}\textcolor{t_darkred}{{\tiny\Vectorstack{-17\\ \pm 2.9}}} & \cellcolor{t_red!25}\textcolor{t_darkred}{{\tiny\Vectorstack{-9.0\\ \pm 3.0}}} & {\tiny\Vectorstack{-3.9\\ \pm 4.1}} &  & \cellcolor{t_red!25}\textcolor{t_darkred}{{\tiny\Vectorstack{-9.8\\ \pm 2.4}}} & \cellcolor{t_red!25}\textcolor{t_darkred}{{\tiny\Vectorstack{-8.2\\ \pm 2.9}}} & \cellcolor{t_red!25}\textcolor{t_darkred}{{\tiny\Vectorstack{-11\\ \pm 2.6}}} & {\tiny\Vectorstack{-4.7\\ \pm 3.1}} & \cellcolor{t_red!25}\textcolor{t_darkred}{{\tiny\Vectorstack{-5.8\\ \pm 2.8}}} \\[2pt]
\textbf{GIN} ($\mathrm{sum}$)&{\tiny\Vectorstack{+3.5\\ \pm 2.2}} & \cellcolor{t_red!25}\textcolor{t_darkred}{{\tiny\Vectorstack{-7.3\\ \pm 2.7}}} & {\tiny\Vectorstack{+0.7\\ \pm 2.7}} & {\tiny\Vectorstack{+5.8\\ \pm 3.6}} & \cellcolor{t_green!25}\textcolor{t_darkgreen}{{\tiny\Vectorstack{+9.8\\ \pm 2.4}}} &  & {\tiny\Vectorstack{+1.6\\ \pm 3.5}} & {\tiny\Vectorstack{-0.9\\ \pm 2.2}} & {\tiny\Vectorstack{+5.1\\ \pm 3.2}} & {\tiny\Vectorstack{+3.9\\ \pm 3.3}} \\[2pt]
\textbf{2-GNN} ($\mean$)&{\tiny\Vectorstack{+1.9\\ \pm 2.9}} & \cellcolor{t_red!25}\textcolor{t_darkred}{{\tiny\Vectorstack{-8.9\\ \pm 2.3}}} & {\tiny\Vectorstack{-0.9\\ \pm 2.0}} & {\tiny\Vectorstack{+4.3\\ \pm 2.5}} & \cellcolor{t_green!25}\textcolor{t_darkgreen}{{\tiny\Vectorstack{+8.2\\ \pm 2.9}}} & {\tiny\Vectorstack{-1.6\\ \pm 3.5}} &  & {\tiny\Vectorstack{-2.4\\ \pm 2.6}} & {\tiny\Vectorstack{+3.5\\ \pm 2.2}} & {\tiny\Vectorstack{+2.4\\ \pm 2.6}} \\[2pt]
\textbf{2-GNN} ($\wmean$)&{\tiny\Vectorstack{+4.4\\ \pm 2.3}} & \cellcolor{t_red!25}\textcolor{t_darkred}{{\tiny\Vectorstack{-6.5\\ \pm 2.0}}} & {\tiny\Vectorstack{+1.6\\ \pm 2.3}} & \cellcolor{t_green!25}\textcolor{t_darkgreen}{{\tiny\Vectorstack{+6.7\\ \pm 2.5}}} & \cellcolor{t_green!25}\textcolor{t_darkgreen}{{\tiny\Vectorstack{+11\\ \pm 2.6}}} & {\tiny\Vectorstack{+0.9\\ \pm 2.2}} & {\tiny\Vectorstack{+2.4\\ \pm 2.6}} &  & {\tiny\Vectorstack{+5.9\\ \pm 3.2}} & {\tiny\Vectorstack{+4.8\\ \pm 3.2}} \\[2pt]
\textbf{2-WL-GNN} ($\mean$)&{\tiny\Vectorstack{-1.6\\ \pm 2.6}} & \cellcolor{t_red!25}\textcolor{t_darkred}{{\tiny\Vectorstack{-12\\ \pm 2.5}}} & {\tiny\Vectorstack{-4.4\\ \pm 2.5}} & {\tiny\Vectorstack{+0.8\\ \pm 2.7}} & {\tiny\Vectorstack{+4.7\\ \pm 3.1}} & {\tiny\Vectorstack{-5.1\\ \pm 3.2}} & {\tiny\Vectorstack{-3.5\\ \pm 2.2}} & {\tiny\Vectorstack{-5.9\\ \pm 3.2}} &  & {\tiny\Vectorstack{-1.1\\ \pm 2.0}} \\[2pt]
\textbf{2-WL-GNN} ($\wmean$)&{\tiny\Vectorstack{-0.4\\ \pm 2.1}} & \cellcolor{t_red!25}\textcolor{t_darkred}{{\tiny\Vectorstack{-11\\ \pm 2.0}}} & {\tiny\Vectorstack{-3.2\\ \pm 2.5}} & {\tiny\Vectorstack{+1.9\\ \pm 3.1}} & \cellcolor{t_green!25}\textcolor{t_darkgreen}{{\tiny\Vectorstack{+5.8\\ \pm 2.8}}} & {\tiny\Vectorstack{-3.9\\ \pm 3.3}} & {\tiny\Vectorstack{-2.4\\ \pm 2.6}} & {\tiny\Vectorstack{-4.8\\ \pm 3.2}} & {\tiny\Vectorstack{+1.1\\ \pm 2.0}} & 
\end{tabular}}

\end{table}
\begin{table}[ht]
	\caption{Fold-wise accuracy delta means and standard deviations on PROTEINS.}\label[mat]{tbl:appendix:diff-proteins}
	\centering\small
{\setlength\tabcolsep{2.5pt}\setlength{\extrarowheight}{2pt}%
\begin{tabular}{lcccccccccc}
& \rotatebox[origin=l]{90}{\textbf{WL\textsubscript{ST}} ($T=3$)} & \rotatebox[origin=l]{90}{\textbf{WL\textsubscript{SP}} ($T=3$)} & \rotatebox[origin=l]{90}{\textbf{2-LWL} ($T=3$)} & \rotatebox[origin=l]{90}{\textbf{2-GWL} ($T=3$)} & \rotatebox[origin=l]{90}{\textbf{Baseline} ($\mathrm{sum}$)} & \rotatebox[origin=l]{90}{\textbf{GIN} ($\mathrm{sum}$)} & \rotatebox[origin=l]{90}{\textbf{2-GNN} ($\mean$)} & \rotatebox[origin=l]{90}{\textbf{2-GNN} ($\wmean$)} & \rotatebox[origin=l]{90}{\textbf{2-WL-GNN} ($\mean$)} & \rotatebox[origin=l]{90}{\textbf{2-WL-GNN} ($\wmean$)} \\
\textbf{WL\textsubscript{ST}} ($T=3$)& & {\tiny\Vectorstack{-0.1\\ \pm 2.8}} & {\tiny\Vectorstack{+3.6\\ \pm 4.4}} & {\tiny\Vectorstack{-0.1\\ \pm 3.3}} & {\tiny\Vectorstack{-1.0\\ \pm 5.1}} & {\tiny\Vectorstack{+1.3\\ \pm 2.1}} & {\tiny\Vectorstack{-1.8\\ \pm 3.6}} & {\tiny\Vectorstack{-0.7\\ \pm 3.4}} & {\tiny\Vectorstack{-3.5\\ \pm 3.1}} & {\tiny\Vectorstack{-2.3\\ \pm 3.7}} \\[2pt]
\textbf{WL\textsubscript{SP}} ($T=3$)&{\tiny\Vectorstack{+0.1\\ \pm 2.8}} &  & {\tiny\Vectorstack{+3.7\\ \pm 3.5}} & {\tiny\Vectorstack{-0.0\\ \pm 3.6}} & {\tiny\Vectorstack{-0.9\\ \pm 4.1}} & {\tiny\Vectorstack{+1.4\\ \pm 3.0}} & {\tiny\Vectorstack{-1.7\\ \pm 4.1}} & {\tiny\Vectorstack{-0.6\\ \pm 4.0}} & {\tiny\Vectorstack{-3.4\\ \pm 3.2}} & {\tiny\Vectorstack{-2.2\\ \pm 3.7}} \\[2pt]
\textbf{2-LWL} ($T=3$)&{\tiny\Vectorstack{-3.6\\ \pm 4.4}} & {\tiny\Vectorstack{-3.7\\ \pm 3.5}} &  & {\tiny\Vectorstack{-3.7\\ \pm 4.6}} & {\tiny\Vectorstack{-4.6\\ \pm 5.7}} & {\tiny\Vectorstack{-2.3\\ \pm 4.0}} & {\tiny\Vectorstack{-5.4\\ \pm 3.1}} & {\tiny\Vectorstack{-4.3\\ \pm 4.6}} & {\tiny\Vectorstack{-7.1\\ \pm 3.8}} & {\tiny\Vectorstack{-5.9\\ \pm 3.4}} \\[2pt]
\textbf{2-GWL} ($T=3$)&{\tiny\Vectorstack{+0.1\\ \pm 3.3}} & {\tiny\Vectorstack{+0.0\\ \pm 3.6}} & {\tiny\Vectorstack{+3.7\\ \pm 4.6}} &  & {\tiny\Vectorstack{-0.9\\ \pm 4.9}} & {\tiny\Vectorstack{+1.4\\ \pm 2.4}} & {\tiny\Vectorstack{-1.7\\ \pm 3.7}} & {\tiny\Vectorstack{-0.6\\ \pm 3.1}} & {\tiny\Vectorstack{-3.4\\ \pm 3.0}} & {\tiny\Vectorstack{-2.2\\ \pm 3.9}} \\[2pt]
\textbf{Baseline} ($\mathrm{sum}$)&{\tiny\Vectorstack{+1.0\\ \pm 5.1}} & {\tiny\Vectorstack{+0.9\\ \pm 4.1}} & {\tiny\Vectorstack{+4.6\\ \pm 5.7}} & {\tiny\Vectorstack{+0.9\\ \pm 4.9}} &  & {\tiny\Vectorstack{+2.3\\ \pm 4.1}} & {\tiny\Vectorstack{-0.8\\ \pm 5.9}} & {\tiny\Vectorstack{+0.3\\ \pm 6.3}} & {\tiny\Vectorstack{-2.5\\ \pm 4.6}} & {\tiny\Vectorstack{-1.3\\ \pm 5.1}} \\[2pt]
\textbf{GIN} ($\mathrm{sum}$)&{\tiny\Vectorstack{-1.3\\ \pm 2.1}} & {\tiny\Vectorstack{-1.4\\ \pm 3.0}} & {\tiny\Vectorstack{+2.3\\ \pm 4.0}} & {\tiny\Vectorstack{-1.4\\ \pm 2.4}} & {\tiny\Vectorstack{-2.3\\ \pm 4.1}} &  & {\tiny\Vectorstack{-3.1\\ \pm 3.0}} & {\tiny\Vectorstack{-2.0\\ \pm 3.0}} & \cellcolor{t_red!25}\textcolor{t_darkred}{{\tiny\Vectorstack{-4.8\\ \pm 2.2}}} & {\tiny\Vectorstack{-3.6\\ \pm 2.8}} \\[2pt]
\textbf{2-GNN} ($\mean$)&{\tiny\Vectorstack{+1.8\\ \pm 3.6}} & {\tiny\Vectorstack{+1.7\\ \pm 4.1}} & {\tiny\Vectorstack{+5.4\\ \pm 3.1}} & {\tiny\Vectorstack{+1.7\\ \pm 3.7}} & {\tiny\Vectorstack{+0.8\\ \pm 5.9}} & {\tiny\Vectorstack{+3.1\\ \pm 3.0}} &  & {\tiny\Vectorstack{+1.1\\ \pm 2.6}} & {\tiny\Vectorstack{-1.7\\ \pm 2.1}} & {\tiny\Vectorstack{-0.5\\ \pm 2.7}} \\[2pt]
\textbf{2-GNN} ($\wmean$)&{\tiny\Vectorstack{+0.7\\ \pm 3.4}} & {\tiny\Vectorstack{+0.6\\ \pm 4.0}} & {\tiny\Vectorstack{+4.3\\ \pm 4.6}} & {\tiny\Vectorstack{+0.6\\ \pm 3.1}} & {\tiny\Vectorstack{-0.3\\ \pm 6.3}} & {\tiny\Vectorstack{+2.0\\ \pm 3.0}} & {\tiny\Vectorstack{-1.1\\ \pm 2.6}} &  & {\tiny\Vectorstack{-2.8\\ \pm 2.2}} & {\tiny\Vectorstack{-1.6\\ \pm 2.8}} \\[2pt]
\textbf{2-WL-GNN} ($\mean$)&{\tiny\Vectorstack{+3.5\\ \pm 3.1}} & {\tiny\Vectorstack{+3.4\\ \pm 3.2}} & {\tiny\Vectorstack{+7.1\\ \pm 3.8}} & {\tiny\Vectorstack{+3.4\\ \pm 3.0}} & {\tiny\Vectorstack{+2.5\\ \pm 4.6}} & \cellcolor{t_green!25}\textcolor{t_darkgreen}{{\tiny\Vectorstack{+4.8\\ \pm 2.2}}} & {\tiny\Vectorstack{+1.7\\ \pm 2.1}} & {\tiny\Vectorstack{+2.8\\ \pm 2.2}} &  & {\tiny\Vectorstack{+1.2\\ \pm 2.1}} \\[2pt]
\textbf{2-WL-GNN} ($\wmean$)&{\tiny\Vectorstack{+2.3\\ \pm 3.7}} & {\tiny\Vectorstack{+2.2\\ \pm 3.7}} & {\tiny\Vectorstack{+5.9\\ \pm 3.4}} & {\tiny\Vectorstack{+2.2\\ \pm 3.9}} & {\tiny\Vectorstack{+1.3\\ \pm 5.1}} & {\tiny\Vectorstack{+3.6\\ \pm 2.8}} & {\tiny\Vectorstack{+0.5\\ \pm 2.7}} & {\tiny\Vectorstack{+1.6\\ \pm 2.8}} & {\tiny\Vectorstack{-1.2\\ \pm 2.1}} & 
\end{tabular}}

\end{table}
\begin{table}[ht]
	\caption{Fold-wise accuracy delta means and standard deviations on D\&D.}\label[mat]{tbl:appendix:diff-dd}
	\centering\small
{\setlength\tabcolsep{2.5pt}\setlength{\extrarowheight}{2pt}%
\begin{tabular}{lcccccccccc}
& \rotatebox[origin=l]{90}{\textbf{WL\textsubscript{ST}} ($T=1$)} & \rotatebox[origin=l]{90}{\textbf{WL\textsubscript{ST}} ($T=3$)} & \rotatebox[origin=l]{90}{\textbf{2-LWL} ($T=3$)} & \rotatebox[origin=l]{90}{\textbf{2-GWL} ($T=3$)} & \rotatebox[origin=l]{90}{\textbf{Baseline} ($\mathrm{sum}$)} & \rotatebox[origin=l]{90}{\textbf{GIN} ($\mathrm{sum}$)} & \rotatebox[origin=l]{90}{\textbf{2-GNN} ($\mean$)} & \rotatebox[origin=l]{90}{\textbf{2-GNN} ($\wmean$)} & \rotatebox[origin=l]{90}{\textbf{2-WL-GNN} ($\mean$)} & \rotatebox[origin=l]{90}{\textbf{2-WL-GNN} ($\wmean$)} \\
\textbf{WL\textsubscript{ST}} ($T=1$)& & {\tiny\Vectorstack{+0.2\\ \pm 1.6}} & {\tiny\Vectorstack{+2.4\\ \pm 3.3}} & {\tiny\Vectorstack{+2.6\\ \pm 3.6}} & {\tiny\Vectorstack{+3.2\\ \pm 2.7}} & {\tiny\Vectorstack{+3.8\\ \pm 2.8}} & {\tiny\Vectorstack{+6.0\\ \pm 6.7}} & {\tiny\Vectorstack{+9.3\\ \pm 6.4}} & {\tiny\Vectorstack{+3.5\\ \pm 4.4}} & {\tiny\Vectorstack{+4.3\\ \pm 4.0}} \\[2pt]
\textbf{WL\textsubscript{ST}} ($T=3$)&{\tiny\Vectorstack{-0.2\\ \pm 1.6}} &  & {\tiny\Vectorstack{+2.2\\ \pm 3.5}} & {\tiny\Vectorstack{+2.5\\ \pm 3.5}} & {\tiny\Vectorstack{+3.1\\ \pm 3.3}} & {\tiny\Vectorstack{+3.6\\ \pm 3.2}} & {\tiny\Vectorstack{+5.9\\ \pm 6.4}} & {\tiny\Vectorstack{+9.1\\ \pm 6.4}} & {\tiny\Vectorstack{+3.4\\ \pm 4.7}} & {\tiny\Vectorstack{+4.1\\ \pm 4.2}} \\[2pt]
\textbf{2-LWL} ($T=3$)&{\tiny\Vectorstack{-2.4\\ \pm 3.3}} & {\tiny\Vectorstack{-2.2\\ \pm 3.5}} &  & {\tiny\Vectorstack{+0.3\\ \pm 1.4}} & {\tiny\Vectorstack{+0.9\\ \pm 3.7}} & {\tiny\Vectorstack{+1.4\\ \pm 4.7}} & {\tiny\Vectorstack{+3.7\\ \pm 6.2}} & {\tiny\Vectorstack{+6.9\\ \pm 5.7}} & {\tiny\Vectorstack{+1.2\\ \pm 4.5}} & {\tiny\Vectorstack{+1.9\\ \pm 4.3}} \\[2pt]
\textbf{2-GWL} ($T=3$)&{\tiny\Vectorstack{-2.6\\ \pm 3.6}} & {\tiny\Vectorstack{-2.5\\ \pm 3.5}} & {\tiny\Vectorstack{-0.3\\ \pm 1.4}} &  & {\tiny\Vectorstack{+0.6\\ \pm 4.3}} & {\tiny\Vectorstack{+1.2\\ \pm 5.2}} & {\tiny\Vectorstack{+3.4\\ \pm 6.9}} & {\tiny\Vectorstack{+6.7\\ \pm 6.3}} & {\tiny\Vectorstack{+0.9\\ \pm 5.2}} & {\tiny\Vectorstack{+1.7\\ \pm 4.8}} \\[2pt]
\textbf{Baseline} ($\mathrm{sum}$)&{\tiny\Vectorstack{-3.2\\ \pm 2.7}} & {\tiny\Vectorstack{-3.1\\ \pm 3.3}} & {\tiny\Vectorstack{-0.9\\ \pm 3.7}} & {\tiny\Vectorstack{-0.6\\ \pm 4.3}} &  & {\tiny\Vectorstack{+0.5\\ \pm 2.3}} & {\tiny\Vectorstack{+2.8\\ \pm 4.9}} & {\tiny\Vectorstack{+6.1\\ \pm 4.7}} & {\tiny\Vectorstack{+0.3\\ \pm 3.4}} & {\tiny\Vectorstack{+1.0\\ \pm 3.2}} \\[2pt]
\textbf{GIN} ($\mathrm{sum}$)&{\tiny\Vectorstack{-3.8\\ \pm 2.8}} & {\tiny\Vectorstack{-3.6\\ \pm 3.2}} & {\tiny\Vectorstack{-1.4\\ \pm 4.7}} & {\tiny\Vectorstack{-1.2\\ \pm 5.2}} & {\tiny\Vectorstack{-0.5\\ \pm 2.3}} &  & {\tiny\Vectorstack{+2.2\\ \pm 5.2}} & {\tiny\Vectorstack{+5.5\\ \pm 5.4}} & {\tiny\Vectorstack{-0.3\\ \pm 3.6}} & {\tiny\Vectorstack{+0.5\\ \pm 3.5}} \\[2pt]
\textbf{2-GNN} ($\mean$)&{\tiny\Vectorstack{-6.0\\ \pm 6.7}} & {\tiny\Vectorstack{-5.9\\ \pm 6.4}} & {\tiny\Vectorstack{-3.7\\ \pm 6.2}} & {\tiny\Vectorstack{-3.4\\ \pm 6.9}} & {\tiny\Vectorstack{-2.8\\ \pm 4.9}} & {\tiny\Vectorstack{-2.2\\ \pm 5.2}} &  & {\tiny\Vectorstack{+3.3\\ \pm 4.1}} & {\tiny\Vectorstack{-2.5\\ \pm 4.7}} & {\tiny\Vectorstack{-1.7\\ \pm 5.0}} \\[2pt]
\textbf{2-GNN} ($\wmean$)&{\tiny\Vectorstack{-9.3\\ \pm 6.4}} & {\tiny\Vectorstack{-9.1\\ \pm 6.4}} & {\tiny\Vectorstack{-6.9\\ \pm 5.7}} & {\tiny\Vectorstack{-6.7\\ \pm 6.3}} & {\tiny\Vectorstack{-6.1\\ \pm 4.7}} & {\tiny\Vectorstack{-5.5\\ \pm 5.4}} & {\tiny\Vectorstack{-3.3\\ \pm 4.1}} &  & \cellcolor{t_red!25}\textcolor{t_darkred}{{\tiny\Vectorstack{-5.8\\ \pm 2.8}}} & {\tiny\Vectorstack{-5.0\\ \pm 3.3}} \\[2pt]
\textbf{2-WL-GNN} ($\mean$)&{\tiny\Vectorstack{-3.5\\ \pm 4.4}} & {\tiny\Vectorstack{-3.4\\ \pm 4.7}} & {\tiny\Vectorstack{-1.2\\ \pm 4.5}} & {\tiny\Vectorstack{-0.9\\ \pm 5.2}} & {\tiny\Vectorstack{-0.3\\ \pm 3.4}} & {\tiny\Vectorstack{+0.3\\ \pm 3.6}} & {\tiny\Vectorstack{+2.5\\ \pm 4.7}} & \cellcolor{t_green!25}\textcolor{t_darkgreen}{{\tiny\Vectorstack{+5.8\\ \pm 2.8}}} &  & {\tiny\Vectorstack{+0.8\\ \pm 1.0}} \\[2pt]
\textbf{2-WL-GNN} ($\wmean$)&{\tiny\Vectorstack{-4.3\\ \pm 4.0}} & {\tiny\Vectorstack{-4.1\\ \pm 4.2}} & {\tiny\Vectorstack{-1.9\\ \pm 4.3}} & {\tiny\Vectorstack{-1.7\\ \pm 4.8}} & {\tiny\Vectorstack{-1.0\\ \pm 3.2}} & {\tiny\Vectorstack{-0.5\\ \pm 3.5}} & {\tiny\Vectorstack{+1.7\\ \pm 5.0}} & {\tiny\Vectorstack{+5.0\\ \pm 3.3}} & {\tiny\Vectorstack{-0.8\\ \pm 1.0}} & 
\end{tabular}}

\end{table}
\begin{table}[ht]
	\caption{Fold-wise accuracy delta means and standard deviations on REDDIT.}\label[mat]{tbl:appendix:diff-reddit}
	\centering\small
{\setlength\tabcolsep{2.5pt}\setlength{\extrarowheight}{2pt}%
\begin{tabular}{lcccccccc}
& \rotatebox[origin=l]{90}{\textbf{WL\textsubscript{ST}} ($T=1$)} & \rotatebox[origin=l]{90}{\textbf{WL\textsubscript{ST}} ($T=3$)} & \rotatebox[origin=l]{90}{\textbf{2-LWL} ($T=3$)} & \rotatebox[origin=l]{90}{\textbf{2-GWL} ($T=3$)} & \rotatebox[origin=l]{90}{\textbf{Baseline} ($\mathrm{sum}$)} & \rotatebox[origin=l]{90}{\textbf{GIN} ($\mathrm{sum}$)} & \rotatebox[origin=l]{90}{\textbf{2-WL-GNN} ($\mean$)} & \rotatebox[origin=l]{90}{\textbf{2-WL-GNN} ($\wmean$)} \\
\textbf{WL\textsubscript{ST}} ($T=1$)& & {\tiny\Vectorstack{-1.8\\ \pm 1.8}} & {\tiny\Vectorstack{+0.5\\ \pm 4.9}} & {\tiny\Vectorstack{+0.8\\ \pm 3.6}} & {\tiny\Vectorstack{+4.2\\ \pm 10}} & {\tiny\Vectorstack{-11\\ \pm 6.9}} & {\tiny\Vectorstack{-7.4\\ \pm 6.2}} & \cellcolor{t_red!25}\textcolor{t_darkred}{{\tiny\Vectorstack{-13\\ \pm 3.5}}} \\[2pt]
\textbf{WL\textsubscript{ST}} ($T=3$)&{\tiny\Vectorstack{+1.8\\ \pm 1.8}} &  & {\tiny\Vectorstack{+2.3\\ \pm 4.7}} & {\tiny\Vectorstack{+2.6\\ \pm 3.5}} & {\tiny\Vectorstack{+5.9\\ \pm 11}} & {\tiny\Vectorstack{-9.0\\ \pm 7.1}} & {\tiny\Vectorstack{-5.7\\ \pm 6.8}} & \cellcolor{t_red!25}\textcolor{t_darkred}{{\tiny\Vectorstack{-11\\ \pm 2.9}}} \\[2pt]
\textbf{2-LWL} ($T=3$)&{\tiny\Vectorstack{-0.5\\ \pm 4.9}} & {\tiny\Vectorstack{-2.3\\ \pm 4.7}} &  & {\tiny\Vectorstack{+0.3\\ \pm 4.2}} & {\tiny\Vectorstack{+3.7\\ \pm 11}} & {\tiny\Vectorstack{-11\\ \pm 7.3}} & {\tiny\Vectorstack{-7.9\\ \pm 5.5}} & \cellcolor{t_red!25}\textcolor{t_darkred}{{\tiny\Vectorstack{-14\\ \pm 3.4}}} \\[2pt]
\textbf{2-GWL} ($T=3$)&{\tiny\Vectorstack{-0.8\\ \pm 3.6}} & {\tiny\Vectorstack{-2.6\\ \pm 3.5}} & {\tiny\Vectorstack{-0.3\\ \pm 4.2}} &  & {\tiny\Vectorstack{+3.3\\ \pm 11}} & {\tiny\Vectorstack{-12\\ \pm 7.6}} & {\tiny\Vectorstack{-8.2\\ \pm 6.2}} & \cellcolor{t_red!25}\textcolor{t_darkred}{{\tiny\Vectorstack{-14\\ \pm 2.7}}} \\[2pt]
\textbf{Baseline} ($\mathrm{sum}$)&{\tiny\Vectorstack{-4.2\\ \pm 10}} & {\tiny\Vectorstack{-5.9\\ \pm 11}} & {\tiny\Vectorstack{-3.7\\ \pm 11}} & {\tiny\Vectorstack{-3.3\\ \pm 11}} &  & {\tiny\Vectorstack{-15\\ \pm 12}} & {\tiny\Vectorstack{-12\\ \pm 13}} & {\tiny\Vectorstack{-17\\ \pm 10}} \\[2pt]
\textbf{GIN} ($\mathrm{sum}$)&{\tiny\Vectorstack{+11\\ \pm 6.9}} & {\tiny\Vectorstack{+9.0\\ \pm 7.1}} & {\tiny\Vectorstack{+11\\ \pm 7.3}} & {\tiny\Vectorstack{+12\\ \pm 7.6}} & {\tiny\Vectorstack{+15\\ \pm 12}} &  & {\tiny\Vectorstack{+3.3\\ \pm 9.6}} & {\tiny\Vectorstack{-2.4\\ \pm 7.0}} \\[2pt]
\textbf{2-WL-GNN} ($\mean$)&{\tiny\Vectorstack{+7.4\\ \pm 6.2}} & {\tiny\Vectorstack{+5.7\\ \pm 6.8}} & {\tiny\Vectorstack{+7.9\\ \pm 5.5}} & {\tiny\Vectorstack{+8.2\\ \pm 6.2}} & {\tiny\Vectorstack{+12\\ \pm 13}} & {\tiny\Vectorstack{-3.3\\ \pm 9.6}} &  & {\tiny\Vectorstack{-5.7\\ \pm 6.5}} \\[2pt]
\textbf{2-WL-GNN} ($\wmean$)&\cellcolor{t_green!25}\textcolor{t_darkgreen}{{\tiny\Vectorstack{+13\\ \pm 3.5}}} & \cellcolor{t_green!25}\textcolor{t_darkgreen}{{\tiny\Vectorstack{+11\\ \pm 2.9}}} & \cellcolor{t_green!25}\textcolor{t_darkgreen}{{\tiny\Vectorstack{+14\\ \pm 3.4}}} & \cellcolor{t_green!25}\textcolor{t_darkgreen}{{\tiny\Vectorstack{+14\\ \pm 2.7}}} & {\tiny\Vectorstack{+17\\ \pm 10}} & {\tiny\Vectorstack{+2.4\\ \pm 7.0}} & {\tiny\Vectorstack{+5.7\\ \pm 6.5}} & 
\end{tabular}}

\end{table}
\begin{table}[ht]
	\caption{Fold-wise accuracy delta means and standard deviations on IMDB.}\label[mat]{tbl:appendix:diff-imdb}
	\centering\small
{\setlength\tabcolsep{2.5pt}\setlength{\extrarowheight}{2pt}%
\begin{tabular}{lcccccccccc}
& \rotatebox[origin=l]{90}{\textbf{WL\textsubscript{ST}} ($T=3$)} & \rotatebox[origin=l]{90}{\textbf{WL\textsubscript{SP}} ($T=3$)} & \rotatebox[origin=l]{90}{\textbf{2-LWL} ($T=3$)} & \rotatebox[origin=l]{90}{\textbf{2-GWL} ($T=3$)} & \rotatebox[origin=l]{90}{\textbf{Baseline} ($\mathrm{sum}$)} & \rotatebox[origin=l]{90}{\textbf{GIN} ($\mathrm{sum}$)} & \rotatebox[origin=l]{90}{\textbf{2-GNN} ($\mean$)} & \rotatebox[origin=l]{90}{\textbf{2-GNN} ($\wmean$)} & \rotatebox[origin=l]{90}{\textbf{2-WL-GNN} ($\mean$)} & \rotatebox[origin=l]{90}{\textbf{2-WL-GNN} ($\wmean$)} \\
\textbf{WL\textsubscript{ST}} ($T=3$)& & {\tiny\Vectorstack{-1.5\\ \pm 4.3}} & {\tiny\Vectorstack{+0.7\\ \pm 2.0}} & {\tiny\Vectorstack{+2.5\\ \pm 3.4}} & \cellcolor{t_green!25}\textcolor{t_darkgreen}{{\tiny\Vectorstack{+22\\ \pm 4.9}}} & {\tiny\Vectorstack{+6.1\\ \pm 6.4}} & {\tiny\Vectorstack{+1.5\\ \pm 2.6}} & {\tiny\Vectorstack{+2.0\\ \pm 2.7}} & {\tiny\Vectorstack{+1.7\\ \pm 3.7}} & {\tiny\Vectorstack{+1.8\\ \pm 3.8}} \\[2pt]
\textbf{WL\textsubscript{SP}} ($T=3$)&{\tiny\Vectorstack{+1.5\\ \pm 4.3}} &  & {\tiny\Vectorstack{+2.2\\ \pm 5.0}} & {\tiny\Vectorstack{+4.0\\ \pm 5.2}} & \cellcolor{t_green!25}\textcolor{t_darkgreen}{{\tiny\Vectorstack{+24\\ \pm 5.9}}} & {\tiny\Vectorstack{+7.6\\ \pm 7.4}} & {\tiny\Vectorstack{+3.0\\ \pm 5.2}} & {\tiny\Vectorstack{+3.5\\ \pm 5.7}} & {\tiny\Vectorstack{+3.2\\ \pm 6.4}} & {\tiny\Vectorstack{+3.3\\ \pm 6.7}} \\[2pt]
\textbf{2-LWL} ($T=3$)&{\tiny\Vectorstack{-0.7\\ \pm 2.0}} & {\tiny\Vectorstack{-2.2\\ \pm 5.0}} &  & {\tiny\Vectorstack{+1.8\\ \pm 2.6}} & \cellcolor{t_green!25}\textcolor{t_darkgreen}{{\tiny\Vectorstack{+22\\ \pm 5.7}}} & {\tiny\Vectorstack{+5.4\\ \pm 7.2}} & {\tiny\Vectorstack{+0.8\\ \pm 2.0}} & {\tiny\Vectorstack{+1.3\\ \pm 3.2}} & {\tiny\Vectorstack{+1.0\\ \pm 3.5}} & {\tiny\Vectorstack{+1.1\\ \pm 3.7}} \\[2pt]
\textbf{2-GWL} ($T=3$)&{\tiny\Vectorstack{-2.5\\ \pm 3.4}} & {\tiny\Vectorstack{-4.0\\ \pm 5.2}} & {\tiny\Vectorstack{-1.8\\ \pm 2.6}} &  & \cellcolor{t_green!25}\textcolor{t_darkgreen}{{\tiny\Vectorstack{+20\\ \pm 5.6}}} & {\tiny\Vectorstack{+3.6\\ \pm 7.1}} & {\tiny\Vectorstack{-1.0\\ \pm 2.6}} & {\tiny\Vectorstack{-0.5\\ \pm 4.1}} & {\tiny\Vectorstack{-0.8\\ \pm 3.1}} & {\tiny\Vectorstack{-0.7\\ \pm 4.3}} \\[2pt]
\textbf{Baseline} ($\mathrm{sum}$)&\cellcolor{t_red!25}\textcolor{t_darkred}{{\tiny\Vectorstack{-22\\ \pm 4.9}}} & \cellcolor{t_red!25}\textcolor{t_darkred}{{\tiny\Vectorstack{-24\\ \pm 5.9}}} & \cellcolor{t_red!25}\textcolor{t_darkred}{{\tiny\Vectorstack{-21\\ \pm 5.7}}} & \cellcolor{t_red!25}\textcolor{t_darkred}{{\tiny\Vectorstack{-20\\ \pm 5.6}}} &  & \cellcolor{t_red!25}\textcolor{t_darkred}{{\tiny\Vectorstack{-16\\ \pm 6.3}}} & \cellcolor{t_red!25}\textcolor{t_darkred}{{\tiny\Vectorstack{-21\\ \pm 6.0}}} & \cellcolor{t_red!25}\textcolor{t_darkred}{{\tiny\Vectorstack{-20\\ \pm 5.6}}} & \cellcolor{t_red!25}\textcolor{t_darkred}{{\tiny\Vectorstack{-21\\ \pm 6.4}}} & \cellcolor{t_red!25}\textcolor{t_darkred}{{\tiny\Vectorstack{-20\\ \pm 6.9}}} \\[2pt]
\textbf{GIN} ($\mathrm{sum}$)&{\tiny\Vectorstack{-6.1\\ \pm 6.4}} & {\tiny\Vectorstack{-7.6\\ \pm 7.4}} & {\tiny\Vectorstack{-5.4\\ \pm 7.2}} & {\tiny\Vectorstack{-3.6\\ \pm 7.1}} & \cellcolor{t_green!25}\textcolor{t_darkgreen}{{\tiny\Vectorstack{+16\\ \pm 6.3}}} &  & {\tiny\Vectorstack{-4.6\\ \pm 7.5}} & {\tiny\Vectorstack{-4.1\\ \pm 7.1}} & {\tiny\Vectorstack{-4.4\\ \pm 7.9}} & {\tiny\Vectorstack{-4.3\\ \pm 8.4}} \\[2pt]
\textbf{2-GNN} ($\mean$)&{\tiny\Vectorstack{-1.5\\ \pm 2.6}} & {\tiny\Vectorstack{-3.0\\ \pm 5.2}} & {\tiny\Vectorstack{-0.8\\ \pm 2.0}} & {\tiny\Vectorstack{+1.0\\ \pm 2.6}} & \cellcolor{t_green!25}\textcolor{t_darkgreen}{{\tiny\Vectorstack{+21\\ \pm 6.0}}} & {\tiny\Vectorstack{+4.6\\ \pm 7.5}} &  & {\tiny\Vectorstack{+0.5\\ \pm 2.8}} & {\tiny\Vectorstack{+0.2\\ \pm 2.8}} & {\tiny\Vectorstack{+0.2\\ \pm 2.4}} \\[2pt]
\textbf{2-GNN} ($\wmean$)&{\tiny\Vectorstack{-2.0\\ \pm 2.7}} & {\tiny\Vectorstack{-3.5\\ \pm 5.7}} & {\tiny\Vectorstack{-1.3\\ \pm 3.2}} & {\tiny\Vectorstack{+0.5\\ \pm 4.1}} & \cellcolor{t_green!25}\textcolor{t_darkgreen}{{\tiny\Vectorstack{+20\\ \pm 5.6}}} & {\tiny\Vectorstack{+4.1\\ \pm 7.1}} & {\tiny\Vectorstack{-0.5\\ \pm 2.8}} &  & {\tiny\Vectorstack{-0.3\\ \pm 3.5}} & {\tiny\Vectorstack{-0.2\\ \pm 2.7}} \\[2pt]
\textbf{2-WL-GNN} ($\mean$)&{\tiny\Vectorstack{-1.7\\ \pm 3.7}} & {\tiny\Vectorstack{-3.2\\ \pm 6.4}} & {\tiny\Vectorstack{-1.0\\ \pm 3.5}} & {\tiny\Vectorstack{+0.8\\ \pm 3.1}} & \cellcolor{t_green!25}\textcolor{t_darkgreen}{{\tiny\Vectorstack{+21\\ \pm 6.4}}} & {\tiny\Vectorstack{+4.4\\ \pm 7.9}} & {\tiny\Vectorstack{-0.2\\ \pm 2.8}} & {\tiny\Vectorstack{+0.3\\ \pm 3.5}} &  & {\tiny\Vectorstack{+0.1\\ \pm 3.4}} \\[2pt]
\textbf{2-WL-GNN} ($\wmean$)&{\tiny\Vectorstack{-1.8\\ \pm 3.8}} & {\tiny\Vectorstack{-3.3\\ \pm 6.7}} & {\tiny\Vectorstack{-1.1\\ \pm 3.7}} & {\tiny\Vectorstack{+0.7\\ \pm 4.3}} & \cellcolor{t_green!25}\textcolor{t_darkgreen}{{\tiny\Vectorstack{+20\\ \pm 6.9}}} & {\tiny\Vectorstack{+4.3\\ \pm 8.4}} & {\tiny\Vectorstack{-0.2\\ \pm 2.4}} & {\tiny\Vectorstack{+0.2\\ \pm 2.7}} & {\tiny\Vectorstack{-0.1\\ \pm 3.4}} & 
\end{tabular}}

\end{table}%
}

\bibliography{literature}

\end{document}